\documentclass[twoside]{article}

\usepackage[accepted]{aistats2025}
\usepackage[utf8]{inputenc} 
\usepackage[T1]{fontenc}    
\usepackage{hyperref}       
\usepackage{url}            
\usepackage{booktabs}       
\usepackage{amsfonts}       
\usepackage{nicefrac}       
\usepackage{microtype}      
\usepackage{xcolor}         

\usepackage[normalem]{ulem}

\useunder{\uline}{\ul}{}
\usepackage{graphicx}
\usepackage{subfigure}
\usepackage{wrapfig,booktabs}
\usepackage[ruled]{algorithm2e}
\usepackage{algpseudocode}
\usepackage[utf8]{inputenc} 
\usepackage[T1]{fontenc}    
\usepackage{hyperref}     
\usepackage{url}            
\usepackage{booktabs}       
\usepackage{amsfonts}       
\usepackage{nicefrac}       
\usepackage{microtype}      
\usepackage{xcolor}         
\usepackage{amsmath}
\usepackage{amsthm}
\usepackage{amssymb}
\usepackage{mathtools}
\usepackage{multirow}
\usepackage{diagbox}
\theoremstyle{plain}
\newtheorem{theorem}{Theorem}[section]
\newtheorem{proposition}[theorem]{Proposition}
\newtheorem{lemma}[theorem]{Lemma}
\newtheorem{corollary}[theorem]{Corollary}
\theoremstyle{definition}
\newtheorem{definition}[theorem]{Definition}
\newtheorem{assumption}[theorem]{Assumption}
\theoremstyle{remark}

\newcommand{\eins}{\boldsymbol{1}}

\usepackage[round]{natbib}

\begin{document}
	
	%
	
	%
	
	\twocolumn[
	
	\aistatstitle{Locally Private Estimation with Public Features}
	
	\aistatsauthor{ Yuheng Ma \And Ke Jia \And Hanfang Yang}
	
	\aistatsaddress{ School of Statistics, \\ Renmin University of China \And  School of Statistics, \\ Renmin University of China \And Center for Applied Statistics,\\School of Statistics, \\ Renmin University of China } ]

	\begin{abstract}
		We initiate the study of locally differentially private (LDP) learning with public features. 
		We define semi-feature LDP, where some features are publicly available while the remaining ones, along with the label, require protection under local differential privacy.
		Under semi-feature LDP, we demonstrate that the mini-max convergence rate for non-parametric regression is significantly reduced compared to that of classical LDP.
		Then we propose \texttt{HistOfTree}, an estimator that fully leverages the information contained in both public and private features. 
		Theoretically, \texttt{HistOfTree} reaches the mini-max optimal convergence rate.
		Empirically, \texttt{HistOfTree} achieves superior performance on both synthetic and real data. 
		We also explore scenarios where users have the flexibility to select features for protection manually. 
		In such cases, we propose an estimator and a data-driven parameter tuning strategy, leading to analogous theoretical and empirical results.
	\end{abstract}

	\section{Introduction}

	Data privacy regulations such as Europe's General Data Protection Regulation (GDPR) \citep{EuropeanParliament2016a} have led to a notable rise in the significance of privacy-preserving machine learning \citep{cummings2018role}.
	As a golden standard, local differential privacy (LDP) \citep{kairouz2014extremal, duchi2018minimax}, a variant of differential privacy (DP) \citep{dwork2006calibrating}, has gained considerable attention in recent years, especially among industry experts \citep{erlingsson2014rappor, apple2017differential}. 
	LDP assumes that each sample is possessed by a data holder, who privatizes their data before it is collected by the curator. 
	Offering a stronger sense of privacy protection compared to central DP, LDP often encounters much more sophisticated challenges which obstruct both the theoretical analysis and practical implementation of LDP learning.

	Fortunately, in some scenarios, a protocol with weakened protection can be adopted to privatize only the sensitive part of each sample. 
	For instance, label differential privacy \citep{ghazi2021deep, malek2021antipodes, badanidiyuru2024optimal} assumes only the labels contain sensitive information, while the features are freely accessible to the data users. 
	Moreover, clear theoretical advantages have been established  \citep{wang2019sparse, xu2023binary, zhao2024theoretical} for label LDP over classical LDP. 
	Recent works demonstrate the effectiveness of partially protecting a small portion of features \citep{curmei2023private, shen2023classification, krichene2024private, chua2024training}, mostly from a methodological perspective. 
	Yet, no theoretical framework has been established under this setting, which is the primary aim of this work.

	As is addressed by GDPR, a properly functioning consent management process should provide granular consent options regarding the specific types of data being collected and the purposes for which it will be processed \citep{nouwens2020dark}. 
	In our case, to meet the requirements, users must be able to choose which parts of their data are collected under privacy protections and which are not. 
	This necessitates the design of estimators with personalized privacy preferences, i.e., the private and public features for each user are manually selected rather than preset.

	Given this context, we study the problem of statistical estimation under local differential privacy with user-personalized public features. 
	We choose non-parametric regression as the task. 
	An extension to classification is straightforward, and our mechanism also directly implies an optimal protocol for density estimation. 
	Our contributions are summarized as follows:
	\begin{itemize}
		\item  We formalize the problem of LDP learning with public features. 
		Specifically, we define semi-feature LDP, where a portion of the features is publicly available, while the remaining ones, along with the label, require protection. 
		Moreover, we establish the first minimax lower bound under semi-feature LDP for non-parametric regression.
		\item We propose the \texttt{HistOfTree} estimator for both aligned (where the position of public features for all users is identical) and personalized (where users select their own public features) privacy preferences.
		Additionally, we provide a data-driven parameter selection rule.
		\item We demonstrate the superiority of \texttt{HistOfTree}. 
		Theoretically, we provide an excess risk upper bound for personalized privacy preferences, indicating a mini-max optimal convergence rate in the aligned case. 
		Empirically, we show that \texttt{HistOfTree} significantly outperforms naive competitors on both synthetic and real datasets.
	\end{itemize}

	\section{Related Work}

	\paragraph{Public Data and Features}
	
	There has been a long line of work about the benefit of public data \citep{papernotsemi, papernot2018scalable, liu2021leveraging, liu2021iterative, yu2021large, yu2022differentially, nasr2023effectively, gu2023choosing, fuentes2024joint, wang2024neural}.
	Public features are less considered. 
	Recently, a line of work has addressed the reasonableness of providing privacy protection w.r.t. a small portion of sensitive features \citep{curmei2023private, shen2023classification, krichene2024private, chua2024training}.
	Effective methodologies were proposed under central differential privacy, although no theoretical framework quantifying their utility has been provided. 
	A more extreme case is label differential privacy 
	\citep{chaudhuri2011sample, busa2021population, ghazi2021deep, malek2021antipodes, cunningham2022geopointgan, ghazi2022regression, badanidiyuru2024optimal, zhao2025enhancing}, where all the features are regarded as public.
	Under central DP, this relaxation is shown to improve the generalization error with only a constant related to privacy budget \citep{badanidiyuru2024optimal}. 
	However, under LDP, the advantage of label LDP over LDP becomes more apparent. 
	See the established generalization gap for both parametric \citep{wang2019sparse} and non-parametric estimation \citep{xu2023binary, zhao2024theoretical}.

	\paragraph{Personalized Privacy}
	
	Our personalized privacy preferences exhibit heterogeneity w.r.t. both users and attributes, i.e., different attributes of the same user can be either private or public, and different users have different choices of which features to protect.
	Various works have studied the problem of learning under heterogeneous privacy preference w.r.t. users. 
	\citet{wang2015personalized, yang2021federated, li2022protecting} considered locally private learning, with each data holder possessing their own privacy budget.
	\citet{sun2014personalized, song2019multiple, zhang2022support, zhang2024adaptive} assume that the privacy budget of each sample is related to its value.
	These studies neither provide theoretical guarantees nor cover our problem setting since there is no heterogeneity w.r.t. different features in their case.
	Similar to our aligned setting, \citet{amorino2023minimax} proposed Componentwise LDP, where each attribute is released through independent private channels with different budgets.
	This is a more restrictive case than ours, as we require private features to be jointly protected.
	They also do not consider heterogeneous privacy w.r.t. users.
	Moreover, their analysis does not allow for public features, i.e., features with infinitely large budgets.
	More recently, \cite{aliakbarpour2024enhancing} proposed a Bayesian Coordinate LDP framework, which is more general than ours. This framework also presents a promising approach to addressing potential risks arising from inner correlations between features, a topic not covered in this paper.

	\section{Methodology}\label{sec::methodology}
	This section is dedicated to the methodology of \texttt{HistOfTree}. 
	In Section \ref{sec:preliminaries}, we first present notations and preliminaries related to regression problems.
	We also provide the definition of semi-feature local differential privacy. 
	Next, we introduce the \texttt{HistOfTree} estimator in Section \ref{sec:ldpanypartition} under aligned privacy preference. 
	In Section \ref{sec:personalizedestimator}, we consider the case where the privacy preference of each user is different and propose a corresponding estimator.

	\subsection{Preliminaries}\label{sec:preliminaries}
	
	\paragraph{Notations}
	Throughout this paper, we use the notation $a_n \lesssim b_n$ and $a_n \gtrsim b_n$ to denote that there exist positive constant $c$ and $c'$ such that $a_n \leq c b_n$ and $a_n \geq c' b_n$, for all $n \in \mathbb{N}$.
	In addition, we denote $a_n\asymp b_n$ if $a_n\lesssim b_n$ and $b_n\lesssim a_n$.
	Let $a\vee b = \max (a,b)$ and $a\wedge b = \min (a,b)$. 
	For any vector $x$, let $x^i$ denote the $i$-th element of $x$. 
	Recall that for $1 \leq p < \infty$, the $L_p$-norm of $x = (x^1, \ldots, x^d)$ is defined by $\|x\|_p := (|x^1|^p + \cdots + |x^d|^p)^{1/p}$. 
	Let $x_{i:j} = (x_i, \cdots, x_j)$ be the slicing view of $x$ in the $i,\cdots, j$ position.
	For any set $A\subset \mathbb{R}^d$, the diameter of $A$ is defined by $\mathrm{diam}(A):=\sup_{x,x'\in A}\|x-x'\|_2$. 
	Let $A\times B$ be the Cartesian product of sets where $A\in\mathcal{X}_1$ and $B\in \mathcal{X}_2$. 
	For measure $\mathrm{P}$ on $\mathrm{X}_1$ and $\mathrm{Q}$ on $\mathcal{X}_2$, define the product measure $\mathrm{P} \otimes \mathrm{Q}$ on $\mathcal{X}_1\times \mathcal{X}_2$ as $\mathrm{P} \otimes \mathrm{Q} (A\times B)= \mathrm{P}(A)\cdot \mathrm{Q}(B)$. For an integer $k$, denote the $k$-fold product measure on $\mathcal{X}_1^k$ as $\mathrm{P}^k$.

	The goal of regression is to predict the value of an unobserved label of a given input $X$, based on a dataset $D:= \{ (X_i, Y_i) \}_{i=1}^n$ consisting of $n$ i.i.d. observations drawn from an unknown probability measure $\mathrm{P}$ on $\mathcal{X}\times \mathcal{Y} =  [0,1]^d \times [-M,M]$.
	It is legitimate to consider the least square loss $L : \mathcal{X} \times \mathcal{Y}\times \mathcal{Y} \to [0, \infty)$ defined by $L(x, y, f(x)) := (y - f(x))^2$ for our target of regression. Then, for a measurable decision function $f : \mathcal{X} \to \mathcal{Y}$, the risk is defined by $\mathcal{R}_{L,\mathrm{P}}(f) := \int_{\mathcal{X} \times \mathcal{Y}} L(x, y, f(x)) \, d\mathrm{P}(x,y)$. The Bayes risk is the smallest possible risk with respect to $\mathrm{P}$ and $L$. The function that achieves the Bayes risk is called {Bayes function}, namely, $f^*(x) := \mathbb{E} \big(Y|X=x\big)$.

	\begin{figure}[!t]
		\centering
		\subfigure[Aligned]{
			\begin{minipage}{0.28\linewidth}
				\centering
				\includegraphics[width=\textwidth]{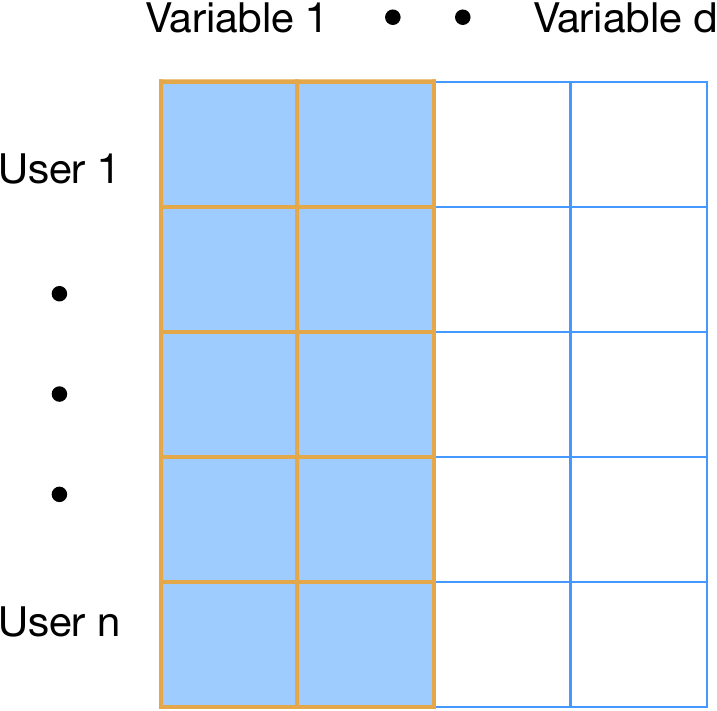}
			\end{minipage}
			\label{fig:diagram1}
		}
		\subfigure[Concentrated]{
			\begin{minipage}{0.28\linewidth}
				\centering
				\includegraphics[width=\textwidth]{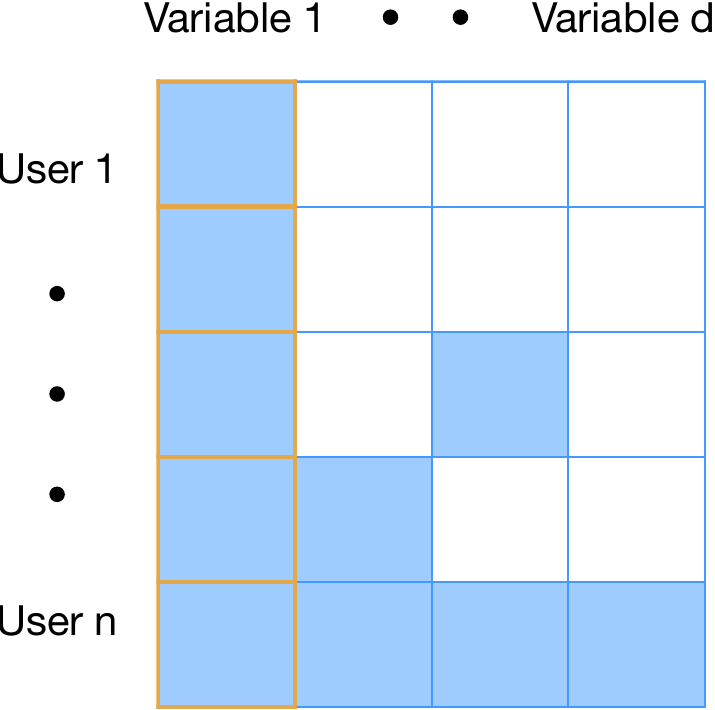}
			\end{minipage}
			\label{fig:diagram2}
		}
		\subfigure[Heavy tailed]{
			\begin{minipage}{0.28\linewidth}
				\centering
				\includegraphics[width=\textwidth]{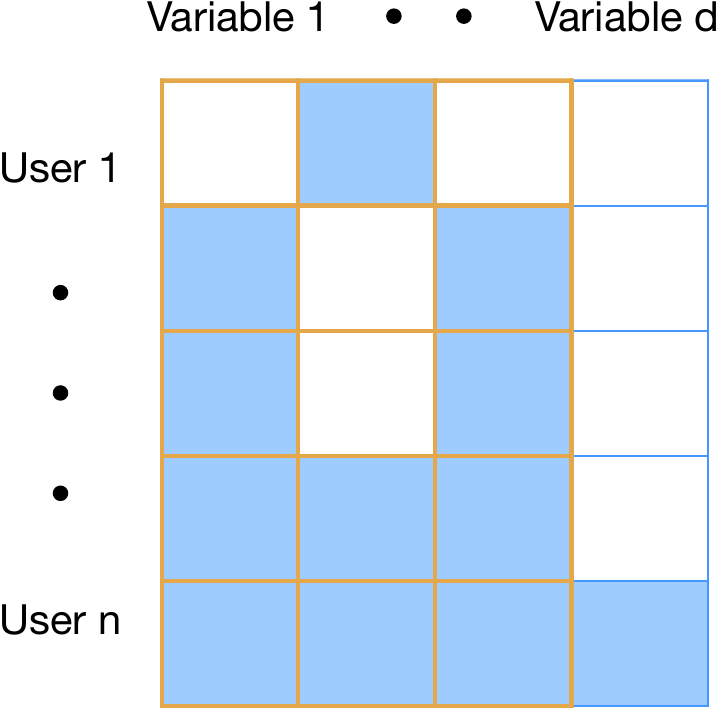}
			\end{minipage}
			\label{fig:diagram3}
		}
		\caption{Illustration of different $W$, where blue means $W_i^j=1$. In the aligned case (a), all users protect the first two features. In the personalized case, users specify different features, with the protected features being concentrated in (b) and spread in (c). The yellow boundaries represent the $s$ selected private features. }
		\label{fig:diagram}
	\end{figure}

	\paragraph{Semi-feature LDP}
	We formally set up private learning with public features. 
	For each $(X_i, Y_i)$, the label $Y_i$ and a subset of features of $X_i$ require privacy protection, while the remaining features can be released freely. 
	Let $W_i^{\ell} \in \{0,1\}$ be the indicator of whether the $\ell$-th feature of $X_i$ should be protected. 
	We assume the collection of private features is independent of $\mathrm{P}$ and $W_i^{\ell} $s are not necessarily protected and are publicly known.  
	The potential risk and solutions are discussed in Section \ref{sec:privacyriskfromw}. 
	We consider the aligned (resp. personalized) case, where the private features are identical (resp. different) for each sample.
	The corresponding $W$ are illustrated in Figure \ref{fig:diagram}. 
	Denote the number of private features of $X_i$ by $s_i^*$.

	With a slight abuse of notation, we write $(X_i, Y_i) = (X_i^{\text {priv}}, X_i^{\text {pub}}, Y_i)$, where $X_i^{\text {priv}}$ and $X_i^{\text {pub}}$ are the collection of private and public  features for $X_i$, respectively. 
	Let $\mathcal{X}^{\text{priv}}_i$ and $\mathcal{X}^{\text{pub}}_i$ be their corresponding domain space.
	Note that $X_i^{\text {priv}}$ and $X_i^{\text {pub}}$ are not necessarily ordered. 
	For instance, we allow $X_i^{\text {priv}} = (X_i^1, X_i^3)$ and $X_i^{\text {pub}} = (X_i^2, X_i^4)$. 
	Given the notations, we rigorously define our privacy constraint.

	\begin{definition}[\textbf{Semi-feature local differential privacy}]\label{def:ldp}
		Given data $\{(X_i,Y_i)\}_{i=1}^n$, a privatized information $Z_i$, which is a random variable on $\mathcal{S}$, is released based on $(X_i, Y_i)$ and $Z_1,\cdots, Z_{i - 1}$. 
		Let $\sigma(\mathcal{S})$ be the $\sigma$-field on $\mathcal{S}$. 
		$Z_i$ is drawn conditional on $(X_i, Y_i)$ and $Z_{1 : i -1}$ via the distribution $\mathrm{R}\left(S \mid X_i^{\text {priv}} , X_i^{\text {pub}}, Y_i , Z_{1 : i -1} \right) $ for $S\in \sigma(\mathcal{S})$. 
		Then the mechanism $\mathrm{R} = \{\mathrm{R}_i\}_{i = 1}^{n}$ is  \textit{$\varepsilon$-semi-feature local differential privacy} (\textit{$\varepsilon$-semi-feature LDP}) if for all $1 \leq i\leq  n, S \in \sigma(\mathcal{S})$, all $x^{\text {priv}},x^{\text {priv}\prime} \in \mathcal{X}^{\text{priv}}_i$, $x^{\text{pub}}  \in \mathcal{X}_i^{\text{pub}}$, $y, y^{\prime} \in \mathcal{Y}$, and $z_{1:i-1} \in \mathcal{S}^{i-1}$, there holds
		\begin{align}\label{equ:defofldpseq}
			\frac{\mathrm{R}_i\left(S \mid x^{\text {priv}},  x^{\text {pub}},  y, z_{1 : i -1} \right)}{\mathrm{R}_i\left(S \mid  x^{\text {priv}\prime},  x^{\text {pub}},  y',  z_{1 : i -1} \right)}  \leq e^{\varepsilon}
		\end{align}
		Moreover, if for all $1 \leq i\leq  n, S \in \sigma(\mathcal{S})$, all $x^{\text {priv}},x^{\text {priv}\prime} \in \mathcal{X}^{\text{priv}}_i$, $x^{\text{pub}}  \in \mathcal{X}_i^{\text{pub}}$, $y, y^{\prime} \in \mathcal{Y}$, there holds
		\begin{align}\label{equ:defofldpnon}
			\frac{\mathrm{R}_i\left(S \mid x^{\text {priv}},  x^{\text {pub}}, y \right)}{\mathrm{R}_i\left(S \mid x^{\text {priv}\prime}, x^{\text {pub}}, y' \right)}  \leq e^{\varepsilon}, 
		\end{align}
		then $\mathrm{R}$ is \textit{non-interactive $\varepsilon$-semi-feature LDP}. 
		Here, we let $0/0 = 1$.
	\end{definition}

	Similar definitions are proposed under central DP setting \citep{shen2023classification, krichene2024private, chua2024training}. 
	Semi-feature LDP requires individuals to guarantee their own privacy by considering the likelihood ratio of each $(X_i^{\text{priv}}, Y_i)$ as in \eqref{equ:defofldpseq} and \eqref{equ:defofldpnon}. 
	The value of $X_i^{\text{pub}}$ is regarded as known.  
	Once a view $z$ is provided, no further processing can reduce the deniability about taking a value $(x^{\text{priv}}, y)$ since any $z$ is nearly as likely to come from other initial value $(x^{\text{priv}\prime}, y^{\prime})$. 
	The definition is meaningful when no prior knowledge exists about the relationship between $x^{\text{priv}}$ and $x^{\text{pub}}$. 
	Otherwise, if $x^{\text{priv}}$ can be inferred using $x^{\text{pub}}$, a mechanism satisfying Definition \ref{def:ldp} may not necessarily provide sufficient protection for $x^{\text{priv}}$. 
	See, for example, \citet{aliakbarpour2024enhancing}. 
	
	If $s_i^* = d$, i.e. all features are private, our definition reduces to the common LDP \citep{kairouz2014extremal, duchi2018minimax} where non-parametric regression has been well-explored \citep{berrett2021strongly, gyorfi2022rate}. 
	If $s_i^* = 0$, i.e. only the labels are private, the definition reduces to label local differential privacy \citep{busa2021population, cunningham2022geopointgan}, which is shown to be an easier problem than pure LDP \citep{wang2019sparse, xu2023binary, zhao2024theoretical}. 
	Our definition explores the intermediate phase, which is practically meaningful \citep{krichene2024private, shen2023classification, chua2024training}.

	\subsection{\texttt{HistOfTree} Estimator for Aligned Private Features}\label{sec:ldpanypartition}

	In this section, we introduce the \texttt{HistOfTree} algorithm under the aligned private features.
	Let $\mathcal{X}^{\text{priv}} = [0,1]^{s^*}$ and $\mathcal{X}^{\text{pub}}= [0,1]^{d - s^*}$ be the private and public domain for all $i$. 
	
	\paragraph{Privacy Mechanism}
	We adopt a partition-based estimator, which creates disjoint partitions and predicts via the average of training labels in each partition grid. 
	Specifically, let $\pi^{\text{priv}} = \{A_j\}$ be a partition of $\mathcal{X}^{\text{priv}}$, i.e. $\cup A_j = \mathcal{X}^{\text{priv}}$ and  $A_{j_1}\cap A_{j_2} = \emptyset$, $j_1\neq j_2$. 
	Similarly, $\pi^{\text{pub}} = \{B_k\}$ on $\mathcal{X}^{\text{pub}}$. 
	The partition-based estimator (which is non-private) of $x$ in  $A_j \times B_k$ is 
	\begin{align}\label{equ:nonprivatedecisiontreeestimator}
		f(x) =  \frac{\sum_{i = 1 }^n Y_i \cdot  \eins_{A_j}(X_i^{\text{priv}}) \eins_{B_k}(X_i^{\text{pub}})  } {\sum_{i = 1 }^n \eins_{A_j}(X_i^{\text{priv}}) \eins_{B_k}(X_i^{\text{pub}}) }.
	\end{align}
	The denominator represents the sample marginal distribution in $A_j \times B_k$, while the numerator corresponds to their joint distribution. Thus, their division estimates their conditional relationship. 
	To perform such an estimation, three pieces of information are necessary from each data holder: $Y_i$, $\eins_{A_j}(X_i^{\text{priv}})$, and $\eins_{B_k}(X_i^{\text{pub}})$. 
	The first two of these require protection. 
	We employ the standard Laplace mechanism \citep{dwork2006calibrating} to protect $Y_i$, specifically:
	\begin{align}\label{equ:privavyprocedureY}
		\tilde{Y}_i = Y_i +  \frac{4M}{\varepsilon} \xi_i, 
	\end{align}
	where $\xi_i$s are i.i.d. standard Laplace random variables. 
	For indicator functions, we use the randomized response mechanism \citep{warner1965randomized} and let 
	\begin{align}\label{equ:privavyprocedureX}
		\tilde{U}_i^j= \begin{cases}C_{\varepsilon} \left(\eins_{A_j}(X_i^{\text{priv}}) -\frac{1}{1+e^{\varepsilon / 4}} \right)& \text { w.p. } \frac{e^{\varepsilon / 4}}{1+e^{\varepsilon / 4}}, \\  C_{\varepsilon} \left(\eins_{A_j^c}(X_i^{\text{priv}}) -\frac{1}{1+e^{\varepsilon / 4}} \right) & \text { w.p. } \frac{1}{1+e^{\varepsilon / 4}}, \end{cases}
	\end{align}
	where  $C_{\varepsilon} = \frac{e^{\varepsilon / 4}+1}{e^{\varepsilon / 4}-1}$. 
	Here, $A_j^c$ denotes the complement of $A_j$. 
	Note that $\mathbb{E}_{ \mathrm{R}}\left[\tilde{U}_{i}^j\right] =\eins_{A_j}(X_i^{\text{priv}})$ and thus we have an unbiased estimator of $\eins_{A_j}(X_i^{\text{priv}})$.

	\paragraph{Partition}
	To formalize the private estimator, we also need to formalize the partitions $\pi^{\text{priv}}$ and $\pi^{\text{pub}}$.
	The choice of $\pi^{\text{priv}}$ is restrictive due to privacy constraints. 
	We adopt the histogram partition, a classical technique in LDP learning \citep{berrett2021strongly, gyorfi2022rate}. 
	Specifically, consider $0 = a_0 < a_1 <\cdots < a_t = 1$, a equal length partition of $[0,1]$.
	Then a histogram partition of $\mathcal{X}^{\text{priv}} = [0,1]^{s^*}$ is $\left\{ \otimes_{k =1}^{s^*} [a_{\sigma_k}, a_{\sigma_k + 1})|0\leq \sigma_k \leq t \right\}$.
	For $\pi^{\text{pub}}$, we can fully utilize the information contained in the dataset.
	Given $\tilde{Y}_i$s and $X^{\text{pub}}_i$s, various candidate methods are available depending on the sample size, computational power, and prior information.
	We choose a decision tree partition for its merit of interpretability, efficiency, stability, extensiveness to multiple feature types, and resistance to the curse of dimensionality. 
	A decision tree partition is obtained by recursively split grids into subgrids using criteria such as variance reduction.
	Since it can be challenging to use the original CART rule \citep{breiman1984classification} for theoretical analysis, we adopt the \textit{max-edge} rule following \citet{cai2023extrapolated, ma2023decision, ma2024optimal}.
	This rule is amenable to theoretical analysis and can also achieve satisfactory practical performance.
	For each grid, the partition rule selects the midpoint of the longest edges that achieves the largest variance reduction. 
	This procedure continues until the depth of the tree reaches its predetermined limit. 
	Together, the \texttt{HistOfTree} estimator is
	\begin{align}\label{equ:ourestimator}
		\tilde{f}(x) = \sum_{A_j, B_k}   \eins_{A_j\times B_k}(x)
		\frac{\sum_{i = 1 }^n  \tilde{Y}_i\tilde{U}_i^j \eins_{B_k}(X_i^{\text{pub}})  } {\sum_{i = 1 }^n \tilde{U}_i^j \eins_{B_k}(X_i^{\text{pub}})}.
	\end{align}

	\subsection{Personalized Private Features}\label{sec:personalizedestimator}

	In practice, it is unlikely that all users share the same preferences regarding which features' privacy to prioritize.
	For instance, while most people may consider their favorite sports as insensitive information, some individuals feel uneasy about releasing such details due to concerns about targeted advertising, particularly from e-commerce platforms.
	We consider $W_i$ to be different for each $i$. 
	We assume the privacy budget $\varepsilon$ are identical across users.

	\textbf{Partition} \;\;
	The first step is to manually select out the features to be treated as private.
	Suppose $\sum_{i=1}^n W_i^{\ell}$ is ranked in decreasing order, i.e., the first feature is the most protected.
	We select out the first $s$ features with the largest $\sum_{i=1}^n W_i^{\ell}$, i.e. $1,\cdots, s$. 
	The selection of parameter $s$ is discussed in Section \ref{sec:selectionofs}.
	We create a histogram partition $\pi^{\text{priv}}$ on these dimensions. 
	Then on the rest of dimensions, we use all public items to create a decision tree partition $\pi^{\text{pub}}$. 
	Specifically, when splitting each grid, the variance reduction along the $\ell$-th dimension \eqref{equ:varreduction} is calculated by samples with $W^{\ell}_i = 0$. 
	We give the detailed algorithm for such partition rule in Algorithm \ref{alg:partition}.

	\begin{algorithm*}[h]
		\caption{Max-edge partition rule}
		\label{alg:partition}
		
		{\bfseries Input: }{ Public data $\{X^{\text{pub}}_i\}_{i=1}^n$, privatized labels $\{\tilde{Y}_i\}_{i=1}^n$, depth $p$. }

		{\bfseries Initialization: } $B_{0,0} =  [0,1]^d$, $D_{0,0} = \{X^{\text{pub}}_i\}_{i=1}^n$, $\pi_k = \emptyset$ for $1\leq k \leq p$. 
		
		\For{$k = 1$ {\bfseries to} $p$}{
			\For{$B_{k-1, j}$ }{
				Suppose $B_{k-1, j} = \times_{\ell=1}^d [a_{\ell}, b_{\ell}]$, let 
				$\mathcal{M}_{i-1}^{j} = \left\{k\mid |b_k - a_k| = \max_{\ell = 1,\cdots,d} |b_{\ell} - a_{\ell}| \right\}$. {\color{blue}\texttt{\# Longest edges.}}  \\
				Let ${B}_{k-1, j}^{0}(\ell) = \left\{x\mid x\in B_{k-1, j}, x^{\ell} < \frac{a_{\ell}+b_{\ell}}{2} \right\}$ and ${B}_{k -1, j}^{1}(\ell) = B_{k-1, j}/ B_{k-1, j}^0(\ell)$.\\ 
				Let ${D}_{k-1, j}(\ell) = \left\{ X^{\text{pub}}_i \in {D}_{k-1, j} \mid W_i^{\ell} = 0\right\}$. 	{\color{blue}\texttt{\# Identify available samples along each axis.}}\\ 
				Let $g(A, D)$ be the sample variance of $\tilde{Y}_i$s whose $X_i^{\text{pub}}$ is in both $A$ and $D$.  Select $\ell$ as
				\begin{align}\label{equ:varreduction}
					{\arg\min}_{\ell \in \mathcal{M}_{i-1}^{j}} \;\; g({B}_{k-1, j}^{0}(\ell), {D}_{k-1, j}(\ell)  ) + g({B}_{k-1, j}^{1}(\ell), {D}_{k-1, j}(\ell) ) .
				\end{align} \\
				${B}_{k, 2j-1} = {B}_{k-1, j}^{0}(\ell)$, ${B}_{k, 2j} = {B}_{k-1, j}^{1}(\ell)$, $\pi_k = \pi_k \cup \{{B}_{k, 2j-1}, {B}_{k, 2j}\}$.
				\\
				$D_{k,2j-1} = \{X^{\text{pub}}_i \in {D}_{k-1, j}(\ell) \mid X^{\text{pub}\ell}_i< \frac{a_{\ell}+b_{\ell}}{2} \}$, $D_{k,2j} = {D}_{k-1, j}(\ell)  / D_{k,2j-1} $. {\color{blue}\texttt{\# Allocate samples. }} 
			}
		}
		{\bfseries Output: }{ Partition $\pi_p$}
	\end{algorithm*}

	\paragraph{Privacy Mechanism}
	Injecting noise into the grid indices is trickier. 
	We first define the potential grids for a sample. 
	For $X_i$, let 
	\begin{align*}
		\mathcal{V}_i = \bigg\{ A \times B \in   \pi^{\text{priv}}\otimes \pi^{\text{pub}} & \mid \exists \overline{X}\in A\times B \; \text{ s.t. }\\ & X_i^{\ell} = \overline{X}^{\ell}  \; \text{ if  } \;\; W_i^{\ell} = 0 \bigg\}. 
	\end{align*}
	Intuitively, $ \mathcal{V}_i$ is the collection of grids where $X_i$ could potentially be located, given no information about its private features. 
	For simplicity, let $j$ be the index of $A \times B$ in the combined partition $\pi^{\text{priv}} \otimes \pi^{\text{pub}}$. 
	We then define the privacy mechanism:
	\begin{equation}\label{equ:privavyprocedurepersonalize}
		\begin{aligned}	\tilde{Y}_i = &  Y_i +  \frac{4M}{\varepsilon} \xi_i,  \text{ and }\\
			\tilde{V}_i^j= &  \begin{cases} 0 & \text{ if } A\times B \notin \mathcal{V}_i,  \\ C_{\varepsilon} \left(\eins_{A\times B}(X_i) -\frac{1}{1+e^{\varepsilon / 4}} \right)& \text {else w.p. } \frac{e^{\varepsilon / 4}}{1+e^{\varepsilon / 4}},  \\  C_{\varepsilon}\left( \eins_{(A\times B)^c}(X_i) -\frac{1}{1+e^{\varepsilon / 4}} \right) & \text {else w.p. } \frac{1}{1+e^{\varepsilon / 4}}.\end{cases}
		\end{aligned}
	\end{equation}
	In this case, the indices of potential grids are estimated analogously to $\tilde{U}_i^j$, while the indices for the remaining grids, which are already known without private features, are zeroed out. 
	Furthermore, we have $\mathbb{E}_{\mathrm{P}, \mathrm{R}}\left[\tilde{V}_{i}^j\right] =\mathbb{E}_{\mathrm{P}}\left[\eins_{A\times B}(X_i)\right]$. 
	Consequently, we define the personalized estimator:
	\begin{align}\label{equ:ourestimatorpersonal}
		\tilde{f}(x) = \sum_{A\times B}   \eins_{A\times B}(x) 
		\frac{\sum_{i = 1 }^n  \tilde{Y}_i\tilde{V}_i^j   } {\sum_{i = 1 }^n \tilde{V}_i^j }
	\end{align}
	where $j$ is the index of $A\times B$. 
	The following proposition shows the privacy guarantee of \eqref{equ:privavyprocedurepersonalize}.
	\begin{proposition}\label{prop:privacy}
		Let $\pi = \{ A\times B\mid A\in\pi^{\text{priv}}, B\in \pi^{\text{pub}}\}$ be any partition of $\mathcal{X}$.
		Then the privacy mechanism \eqref{equ:privavyprocedurepersonalize} is non-interactively $\varepsilon$-semi-feature LDP. 
	\end{proposition}
	When the public features are shared across users, i.e. $\sum_{i = 1}^n W_i^{\ell} = n $ for $1\leq \ell \leq s $ and $\sum_{i = 1}^n W_i^{\ell} =0 $ otherwise, mechanism \eqref{equ:privavyprocedurepersonalize} reduces to that in \eqref{equ:privavyprocedureY} and \eqref{equ:privavyprocedureX}.
	The privacy guarantees of these mechanisms are thus implied by Proposition \ref{prop:privacy}. 
	In this case, the estimator \eqref{equ:ourestimatorpersonal} also recovers \eqref{equ:ourestimator}.

	\subsection{Potential Privacy Risk from $W$}\label{sec:privacyriskfromw}
	
	In some cases, $W$ and $(X,y)$ are dependent, which brings additional privacy leakage.
	For instance, AIDS carriers may choose to conceal their status if they are diagnosed with the condition.
	If the carriers are aware of the mechanism and believes it could leak information, they would provide incorrect results (choose not to prevent this feature and claim no AIDS).
	We propose two solutions. 
	(i)
	If the curator wishes to prevent such wrong data collection, it should classify highly correlated features as private beforehand, including these features among the first $s^*$ private features. 
	This approach requires certain prior knowledge and does not completely resolve the privacy accounting issue.
	(ii) Observing the estimator \eqref{equ:ourestimatorpersonal}, there is no need to release $W$ to the server. The server only needs to receive a binary vector of $V$.
	The only remaining issue is the tree partition, which could be mitigated by using random selection instead of \eqref{equ:varreduction}.
	Although this reduces the effectiveness of the partition, it entirely eliminates the risk associated with $W$.

	\section{Theoretical Results}\label{sec:theoretical}
	
	In this section, we present our theoretical results and related comments. 
	We first provide the mini-max lower bound under semi-feature LDP with aligned privacy preference in Section \ref{sec:lowerbound}. 
	In Section \ref{sec:utilityguarantee}, we establish the optimal convergence rate of \texttt{HistOfTree} estimator. 
	Based on the theoretical findings, we discuss the choice of the number of private dimensions $s$ in Section \ref{sec:selectionofs}.

	\subsection{A Mini-max Lower Bound}\label{sec:lowerbound}

	We first present a necessary assumption on the distribution $\mathrm{P}$, which is a standard condition widely used in non-parametric statistics.

	\begin{assumption}\label{asp:alphaholder}
		\label{asp:boundedmarginal}
		Assume that the density function of $\mathrm{P}$ is upper and lower bounded, i.e. $\underline{c} \leq {d\mathrm{P} (x)} / {dx}\leq \overline{c}$ for some $\overline{c} \geq  \underline{c}>0$. 
		Let $\alpha \in (0, 1]$. Assume  $f^* : \mathcal{X} \to \mathbb{R}$ is $\alpha$-H\"{o}lder continuous, i.e. there exists  $c_L > 0$ s.t. for all $x_1, x_2 \in \mathcal{X}$,  $|f^*(x_1) - f^*(x_2)| \leq c_L \|x_1 - x_2\|^{\alpha}$.
	\end{assumption}
	
	The bounded density assumption can greatly simply the analysis, yet can be removed without hurting the conclusion, see \citet{gyorfi2022rate}.
	We present the following theorem that specifies the mini-max lower bound with semi-feature LDP  under the above assumptions.
	The proof is based on first constructing a finite class of hypotheses and then applying information inequalities for local privacy \citep{duchi2018minimax} and Assouad's Lemma \citep{tsybakov2008introduction}.

	\begin{theorem}\label{thm:lowerbound}
		Denote the function class of $\mathrm{P}$ satisfying Assumption \ref{asp:alphaholder} by $\mathcal{F}$. 
		Consider the aligned privacy preference where $W_i^{\ell} = 1$ if $\ell \leq s^*$ and $W_i^{\ell} = 0$ otherwise. 
		Then for any estimator $\widehat{f}$ that is  sequentially-interactive $\varepsilon$-semi-feature LDP, there holds
		\begin{align}\nonumber
			&	\inf_{\widehat{f}}\sup_{\mathcal{F}}\mathbb{E}_{\mathrm{P}}\left[ \mathcal{R}_{\mathrm{P}} (\widehat{f}) - \mathcal{R}_{\mathrm{P}}^*  \right] \\
			\gtrsim &  \left(n (e^\varepsilon-1)^2\right)^{-\frac{2 \alpha}{2\alpha+d + s^* }}  \vee n^{-\frac{2\alpha}{2\alpha + d}}.
			\label{equ:minimaxconvergencerate}
		\end{align}
	\end{theorem}
	
	The lower bound consists of two terms. 
	The second term, $n^{-\frac{2\alpha}{2\alpha + d}}$, represents the classical mini-max lower bound for non-private learners under Assumption \ref{asp:alphaholder} \citep{tsybakov2008introduction}. 
	Regarding the first term, if $\varepsilon$ is sufficiently large such that $\left(e^{\varepsilon} - 1\right) \gtrsim n^{\frac{s^*}{4\alpha + 2d}}$, it is dominated by the second term. 
	This is the case where the level of privacy is not significant enough to degrade the estimator. 
	For constant level $\varepsilon$, it approaches the second term as $s^*$ becomes smaller, as the learning essentially becomes non-private for $s^* = 0$. 
	If $s^* = d$, our results reduce to the special case of LDP non-parametric regression with only private data. 
	Previous studies \citep{berrett2021strongly, gyorfi2022rate} have provided estimators that achieve the mini-max optimal convergence rate. 
	Our results depict the optimal behavior in the intermediate zone. 
	The performance gain brought by $d-s^*$ public features grows with the rate $\frac{2\alpha}{2\alpha + 2d - (d - s^*)}$, which becomes more significant for larger values of $d-s^*$.

	\subsection{Convergence Rate of \texttt{HistOfTree} }\label{sec:utilityguarantee}
	
	In this section, we present an upper bound for the \texttt{HistOfTree} estimator under personalized privacy preferences. 
	This bound subsequently implies the optimal rate under aligned privacy preferences.

	\begin{theorem}\label{thm:utility}
		Let the \texttt{HistOfTree} estimator $\tilde{f}$ be defined in \eqref{equ:ourestimatorpersonal}. 
		For $\pi = \pi^{\text{priv}} \otimes \pi^{\text{pub}} $, let $\pi^{\text{priv}}$ be a histogram partition with $t$ bins and $\pi^{\text{pub}}$ be generated by Algorithm \ref{alg:partition} with depth $p$.  
		Let Assumption \ref{asp:alphaholder} holds. 
		Define the solution to the equation 
		\begin{align}\label{equ:selectionofp}
			p^* = \underset{p}{\arg\min} \frac{ 2^{p \frac{d + s}{d - s}}  \cdot \log n}{n\varepsilon^2 } \cdot \delta(p) + \; 2^{-2 \alpha p / (d - s)} 
		\end{align} 
		where we let 
		\begin{align}
			\delta(p) =    \frac{1}{n}\sum_{i=1}^n 2^{\sum_{\ell=s+1}^d W_i^{\ell} \cdot p /(d-s)} . 
		\end{align}
		Then, let $\lambda^* = \log_2 (\delta(p^*)) / p^{*}$. 
		Then, $\tilde{f}$ is $\varepsilon$-semi-feature LDP. 
		Moreover, for any $\varepsilon \lesssim \left(n / \log n\right)^{\frac{s}{\alpha + d - s}}$, there exists some choice of the parameters $p\asymp p^* \asymp \log n\varepsilon^2$ and $t \asymp 2^{p^* / (d - s)}$ such that 
		\begin{align}\label{equ:convergencerateofpersonalizedestimator}
			\mathcal{R}_{L,\mathrm{P}} (\tilde{f}) - \mathcal{R}_{L,\mathrm{P}}^*   \lesssim  \left(\frac{\log n
			}{n \varepsilon^2}\right)^{\frac{2 \alpha}{2\alpha+d + s + \lambda^*(d - s)}}
		\end{align}
		holds with probability $1- 8 /n^2$ w.r.t. $\mathrm{P}^n\otimes\mathrm{R}$ where $\mathrm{R}$ is the joint distribution of privacy mechanisms in \eqref{equ:privavyprocedurepersonalize}.
	\end{theorem}
	
	The upper bound \eqref{equ:convergencerateofpersonalizedestimator} is some power of $\log n / n\varepsilon^2$. 
	The presence of the $\log n$ term comes from the high probability nature of Theorem \ref{thm:utility}, while the privacy budget contributes a term of $\varepsilon^2$. 
	Apparently, the tighter the privacy requirement (i.e., the smaller $\varepsilon$), the more relaxed the upper bound becomes.
	Moreover, we require $\varepsilon \lesssim \left(n / \log n\right)^{\frac{s}{\alpha + d - s}}$.
	If $\varepsilon$ exceeds this threshold, the privacy noise becomes negligible, and the presented upper bound might be dominated by the non-private term $n^{-\frac{2\alpha}{2\alpha + d}}$.
	
	We comment on the theorem regarding the selected $p^*$ and the term $\delta(p^*)$. 
	The final rate \eqref{equ:convergencerateofpersonalizedestimator} involves $\lambda^*$, a manufactured quantity representing the impact of the denser tail of the privacy mask matrix $W$. 
	Generally, for a fixed $p$, if a user $i$ chooses to protect more features in $s+1 ,\cdots, d$, the quantity $\sum_{\ell = s+1}^{d} W^{\ell}_i$ becomes larger, and $\delta(p)$ also increases.
	In this case, the selected $p^*$ decreases, whereas the final $\delta(p^*)$ increases (see proof for details). 
	This suggests that, for more demanding privacy protection, the optimal depth $p^*$ (as well as the number of histogram bins $t$, since $ t \asymp \log n\varepsilon^2$) should be reduced, meaning fewer bins are preferable to maintain estimation stability.
	Additionally, since $2^{\lambda^* p^*}  = \delta(p^*)$, the value of $\lambda^*$ is larger, subsequently increasing the upper bound.
	This confirms the intuition, as protecting more features inherently necessitates a looser bound.

	Under the aligned privacy preference, where $W_i^{\ell} = 0$ for $\ell \geq s^* + 1$, Theorem \ref{thm:utility} has a direct implication as outlined in the following corollary.

	\begin{corollary}
		Consider the aligned privacy preference where $W_i^{\ell} = 1$ if $\ell \leq s^* $ and $W_i^{\ell} = 0$ otherwise. 	
		Then, under the same conditions and settings to Theorem \ref{thm:utility}, let $p\asymp p^* \asymp \log n\varepsilon^2$ and $t \asymp 2^{p^* / (d - s^* )}$.
		There holds
		\begin{align}\label{equ:convergencerateofalignedestimator}
			\mathcal{R}_{L,\mathrm{P}} (\tilde{f}) - \mathcal{R}_{L,\mathrm{P}}^*   \lesssim  \left(\frac{\log n
			}{n \varepsilon^2}\right)^{\frac{2 \alpha}{2\alpha+d + s^*  }}
		\end{align}
		with probability $1- 8 /n^2$ with respect to $\mathrm{P}^n\otimes\mathrm{R}$.
	\end{corollary}

	Compared to Theorem \ref{thm:lowerbound}, the \texttt{HistOfTree} estimator with the best parameter choice reaches the mini-max lower bound. 
	We note that when $\varepsilon$ is large, there is a gap between $(e^{\varepsilon} - 1)^2$ as in the lower bound \eqref{equ:minimaxconvergencerate} and
	$\varepsilon^2$ in the upper bound.
	The gap is commonly observed \citep{duchi2018minimax, gyorfi2022rate, ma2023decision, ma2024optimal, ma2024better, zhao2024learning}.

	\subsection{Selection of $s$}\label{sec:selectionofs}

	In this section, we explore the selection of the parameter  $s$ within the context of personalized privacy preferences. 
	Unlike scenarios with aligned preferences, where an intrinsic value of $s = s^* $ is implied, the personalized case requires the explicit specification of  $s$ as a hyper-parameter. 
	When considering the final convergence rate \eqref{equ:convergencerateofpersonalizedestimator}, it is noted that choosing either a significantly large or small value for $s$ inevitably leads to a large denominator of the rate,  impacting the efficiency of the convergence.
	From the proof of Theorem \ref{thm:utility}, we recognize that the right-hand side of \eqref{equ:selectionofp} is a high probability upper bound of the excess risk. 
	As a result, we tend to choose the $s$ that minimizes the upper bound. 
	Together with the selection of $p$, the parameter tuning is specified as an optimization problem
	\begin{align}\label{equ:obejectforparametertuning}
		s, p^* = \underset{s, p}{\arg\min} \frac{ 2^{p \frac{d + s}{d - s}}  \cdot \log n}{n\varepsilon^2 } \cdot \delta(p) + \; 2^{-2  p / (d - s)}
	\end{align}
	where the unknown $\alpha$ is replaced by 1, i.e. the Lipschitz continuous case. 
	Since there are at most $d\log n$ possible candidates for $(s, p^*)$, this can be done via a brute force search. 
	The minimum, analogous to Theorem \ref{thm:utility}, is $ \left({\log n
	}/{n \varepsilon^2}\right)^{\frac{2 }{2+d + s + \lambda^*(d - s)}}$. 
	Clearly, if the private features are concentrated across users, i.e. most $\sum_{\ell=s+1}^d W_i^{\ell}$ are small, $\lambda^*$ will eventually be small.
	As a result, $s$ leans to be small. 
	In the opposite, if private features are less concentrated (long tail), a large $s$ is selected to ensure a broader privacy coverage. 
	This behavior is illustrated in Figure \ref{fig:diagram}.

	There are also some interesting corner cases of the selection procedure. 
	First, the solution is consistent with the aligned case. 
	Specifically, selecting any $s'>s^*$ leads to $\lambda^* =0$.
	As a result, the rate has $\frac{2 }{2+d + s' } < \frac{2 }{2+d + s^* }$. 
	For $s' < s^*$, the rate remains $ \frac{2 }{2+d + s^*}$, meaning that choosing small $s'$ is theoretically equivalent, although it may result in less effective partitions in practice.
	\eqref{equ:obejectforparametertuning} also covers the case of locally private estimation with public data \citep{ma2023decision}, where $W_i^{\ell} = 1$ if $i \leq n - n_{q}$ and $W_i^{\ell} = 0$ otherwise. 	
	In this case, the object in \eqref{equ:obejectforparametertuning} yields a $\lambda^*$ strictly smaller than 1, and $s = 0$ is the optimal choice. 
	Our estimator reduces to the locally private decision tree proposed by \citet{ma2023decision}.

	\section{Experiments}\label{sec:experiments}
	
	To demonstrate the superiority of proposed methods and to validate our theoretical findings, we conduct experiments on both synthetic and real datasets in Section \ref{sec:simulation} and \ref{sec:realdata}, respectively. 
	
	The tested methods include: 
	\textbf{(i)} \texttt{HistOfTree} is the proposed estimator. 
	In addition to the partition rule proposed in Algorithm \ref{alg:partition}, we boost the performance by incorporating the criterion reduction scheme from the original CART \citep{breiman1984classification}, similar to \citet{ma2023decision, ma2024optimal};
	\textbf{(ii)} \texttt{AdHistOfTree} is identical to \texttt{HistOfTree} except that the choice of $s$, $p$, and $t$ are selected according to Section \ref{sec:selectionofs}. 
	\textbf{(iii)} \texttt{Hist} is the locally private histogram estimation \citep{berrett2021strongly, gyorfi2022rate} which treats all features and the label as private;
	\textbf{(iv)} \texttt{KRR} is the obfuscation approach that directly perturb $X_i^j$ using $k$-generalized randomized response if $W_i^j = 1$. 
	It then fit with a decision tree. 
	\textbf{(v)} \texttt{ParDT} estimates by adding Laplace noise to labels and then fitting a decision tree based on the noisy label and the public part of the samples. 
	\textbf{(vi)}  \texttt{LabelDT} estimates by adding Laplace noise to labels and then using a decision tree. 
	It serves as the label LDP benchmark.
	\textbf{(vi)} \texttt{DT} is the non-private decision tree and serves as the non-private benchmark. 
	Implementation details of all these methods are provided in Appendix \ref{app:experiments}. 
	
	The evaluation metric in all experiments is the mean squared error (MSE).
	For each model, we report the best result over its parameter grids, with the best result determined based on the average of at least 50 replications. 
	The size of the parameter grids is selected based on running time to ensure that each method incurs an equal amount of computation. 
	All experiments are conducted on a machine with 72-core Intel Xeon 2.60GHz and 128GB of main memory.
	The illustrative codes can be found on GitHub\footnote{\url{https://github.com/Karlmyh/LDP-PublicFeatures}}.

	\begin{figure}[!h]
		\centering
		\subfigure[Privacy utility trade off]{
			\begin{minipage}{0.45\linewidth}
				\centering
				\includegraphics[width=\textwidth]{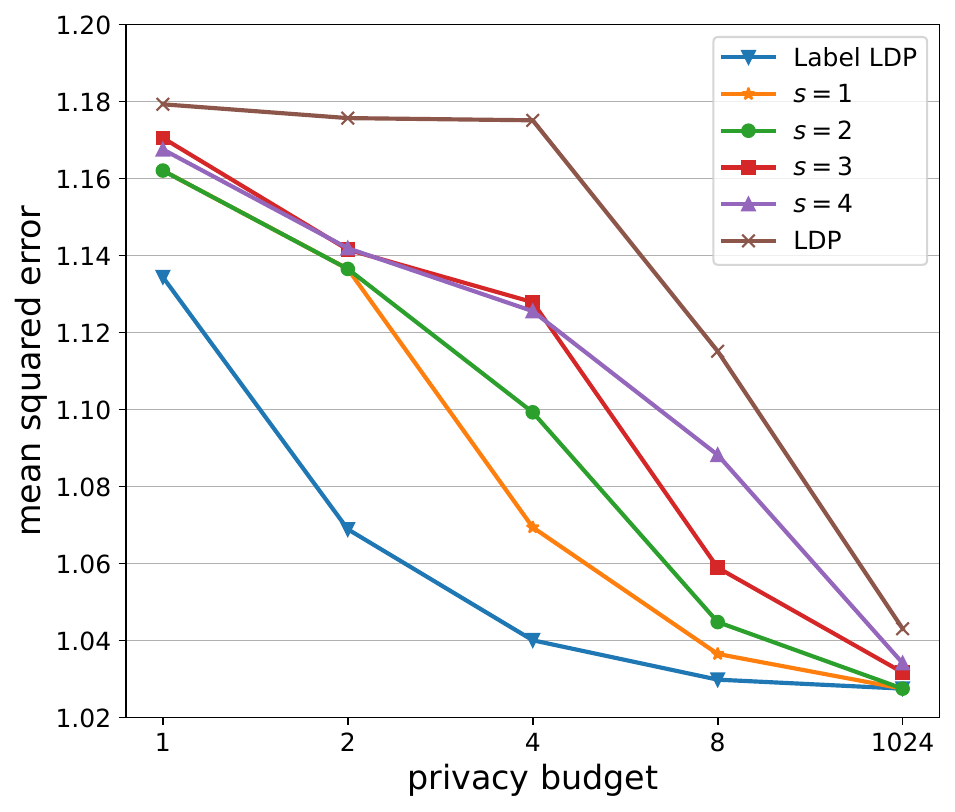}
			\end{minipage}
			\label{fig:privacyutility}
		}
		\subfigure[Consistency]{
			\begin{minipage}{0.45\linewidth}
				\centering
				\includegraphics[width=\textwidth]{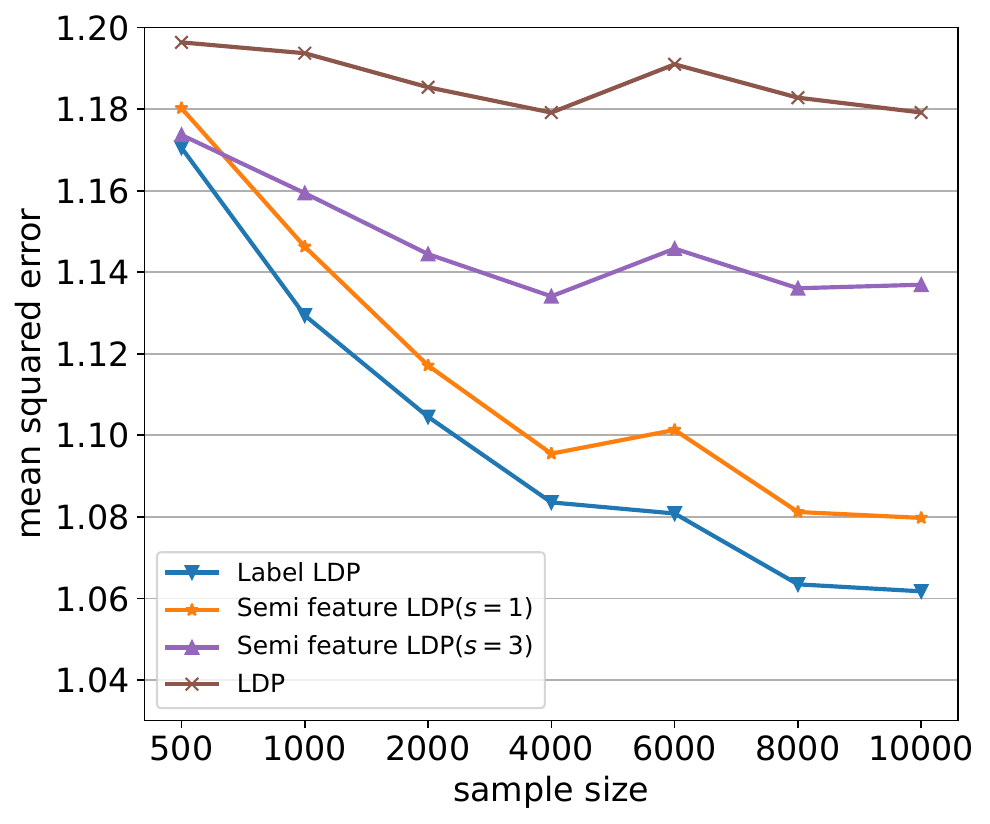}
			\end{minipage}
			\label{fig:samplesize}
		}
		\subfigure[Parameter analysis]{
			\begin{minipage}{0.45\linewidth}
				\centering
				\includegraphics[width=\textwidth]{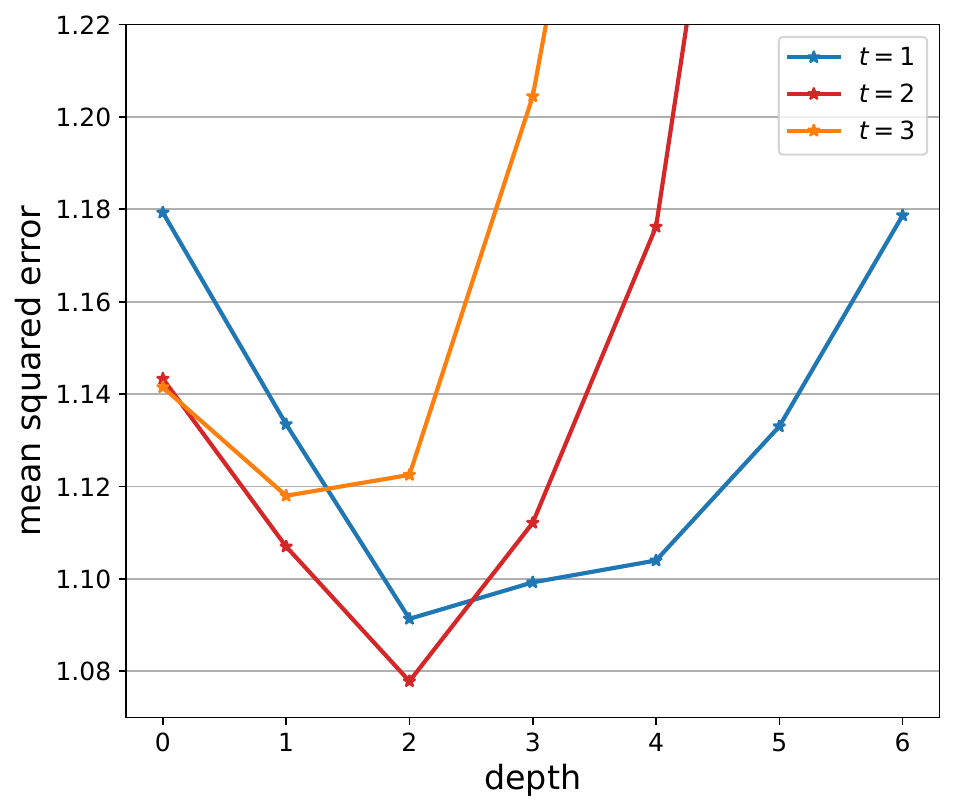}
			\end{minipage}
			\label{fig:parameter}
		}
		\subfigure[Selection of $s$]{
			\begin{minipage}{0.45\linewidth}
				\centering
				\includegraphics[width=\textwidth]{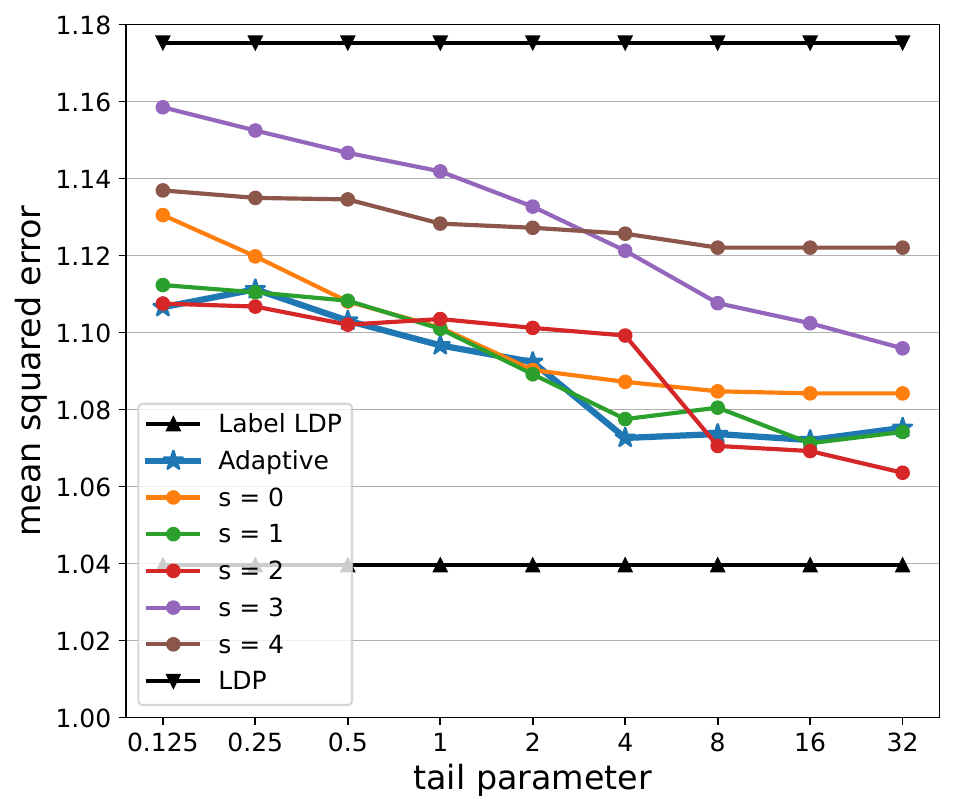}
			\end{minipage}
			\label{fig:selects}
		}
		\caption{Experiment results on synthetic data. \texttt{LabelDT} and \texttt{Hist} are captioned as label LDP and LDP, respectively. \texttt{HistOfTree} is captioned with specific choice of parameters $s$ and $t$. 
			In \ref{fig:privacyutility}, we apply uneven scaling to the x-axis to accommodate the outlying value of 1024, representing the non-private performance.	
			In \ref{fig:selects}, \texttt{AdHistOfTree} is captioned as adaptive.
		}
		\label{fig::simulation}
	\end{figure}

	\subsection{Simulation}\label{sec:simulation}

	In the simulation, we consider the following distribution $\mathrm{P}$. 
	The marginal distributions $\mathrm{P}_X $ is the uniform distributions on $\mathcal{X} = [0,1]^5$.
	The label is generated via $Y = \sin(X^1) + \sin(X^2) + \sin(X^5) + \sigma$, where $\sigma$ is a standard Gaussian random variable. 
	It can be easily verified that the constructed distributions above satisfy Assumption \ref{asp:alphaholder}. 
	For mask matrix $W$, we always let $\sum_{i=1}^nW_i^{\ell} \geq \sum_{i=1}^nW_i^{\ell'}$ if $\ell \leq \ell'$. 
	In this case, useful information is contained in both public and private features, which necessitates our method.

	\begin{table*}[!h]
		\centering
		\caption{Real data performance for $\varepsilon = 2$. 
			To ensure significance, we employ the Wilcoxon signed-rank test \citep{wilcoxon1992individual} with a significance level of 0.05 to determine if a result is significantly better. 
			The best results are \textbf{bolded} and those holding significance towards the rest results are marked with $*$.
			ME and CART stand for max-edge and the classical CART rule, respectively. }
		\label{tab:realdataepsilon2}
		\resizebox{0.95\linewidth}{!}{
			\renewcommand{\arraystretch}{1}
			\setlength{\tabcolsep}{8pt}
			\begin{tabular}{l|lllll|lllllll}
				\toprule
				\multicolumn{1}{c|}{\multirow{3}{*}{}}                                  & \multicolumn{5}{c|}{Aligned}                                                                                                                                  & \multicolumn{7}{c}{Personalized}                                                                                                                                                                                 \\ 
				\cmidrule(r){2-6} \cmidrule(l){7-13}
				\multicolumn{1}{c|}{}                                                   & \multicolumn{2}{c}{\texttt{HistOfTree}}           & \multicolumn{1}{c}{\multirow{2}{*}{\texttt{ParDT}}} & \multicolumn{1}{c}{\multirow{2}{*}{\texttt{Hist}}} &  \multicolumn{1}{c|}{\multirow{2}{*}{\texttt{KRR}}} & \multicolumn{2}{c}{\texttt{HistOfTree}}           & \multicolumn{2}{c}{\texttt{AdHistOfTree}}         & \multicolumn{1}{c}{\multirow{2}{*}{\texttt{ParDT}}} & \multicolumn{1}{c}{\multirow{2}{*}{\texttt{Hist}}}  & \multicolumn{1}{c}{\multirow{2}{*}{\texttt{KRR}}} \\
				\multicolumn{1}{c|}{}                                                   & \multicolumn{1}{c}{ME} & \multicolumn{1}{c}{CART} & \multicolumn{1}{c}{}                                & \multicolumn{1}{c}{}            &                   & \multicolumn{1}{c}{ME} & \multicolumn{1}{c}{CART} & \multicolumn{1}{c}{ME} & \multicolumn{1}{c}{CART} & \multicolumn{1}{c}{}                                & \multicolumn{1}{c}{}                               \\ \midrule
				ABA                                                                     & 1.85                   & 1.65                     & \textbf{1.63}                                                & 1.98                    & 1.81                            & \textbf{1.61*}                 & 1.80                     & 1.61                   & 1.78                     & 1.65                                                & 1.98                &                 1.77                \\
				AQU                                                                     & 1.55                   & \textbf{1.47*}                    & 1.71                                                & 1.63                      & 2.05                          & \textbf{1.60*}                   & 1.60                     & \textbf{1.60*}                   & 1.60                     & 1.75                                                & 1.63               & 2.05                                \\
				BUI                                                                     & 1.48                   & \textbf{1.27*}                     & 1.51                                                & 1e4                       & 2.43                          & \textbf{1.42}                   & 1.47                     & 1.43                   & 1.47                     & 1.59                                                & 1e4                          & 2.41                      \\
				DAK                                                                     & 6.69                   & \textbf{6.65*}                     & 8.03                                                & 8.34                        & 9.57                        & 6.73                   & 6.73                     & \textbf{6.57*}                   & \textbf{6.57*}                     & 8.85                                                & 8.34                  & 9.57                             \\
				FIS                                                                     & \textbf{1.88*}                   & 2.01                     & 2.14                                                & 2.16                            & 2.35                    & 2.09                   & 2.16                     & \textbf{2.07}                   & 2.14                     & 2.23                                                & 2.16                    & 2.42                           \\
				MUS                                                                     & \textbf{1.13*}                   & \textbf{1.13*}                     & 1.26                                                & 2e5                        & 1.44                         & 1.13                   & 1.13                     & \textbf{1.12}                  & 1.13                     & 1.26                                                & 2e5               & 1.45                                 \\
				POR                                                                     & 3.15                   & \textbf{3.04*}                     & 3.32                                                & 4.23                      & 4.02                          & 3.08                   & 3.08                     & \textbf{2.98*}                   & \textbf{2.98*}                     & 3.35                                                & 4.23                    & 3.92                           \\
				PYR                                                                     & \textbf{1.48*}                   & \textbf{1.48*}                    & 1.65                                                & 3e3                          & 2.29                       & 1.49                   & 1.49                     & \textbf{1.42*}                   & \textbf{1.42*}                     & 1.77                                                & 3e3                     & 2.30                           \\
				RED                                                                     & 1.44                   & \textbf{1.37*}                    & 1.45                                                & 1.55                 & 1.67                               & 1.44                   & \textbf{1.43}                     & 1.44                   & 1.44                     & 1.46                                                & 1.55                     & 1.67                          \\
				WHI                                                                     & 1.42                   & \textbf{1.41}                     & \textbf{1.41}                                                & 1.53                        & 1.50                        & 1.42                   & \textbf{1.39*}                     & 1.44                   & 1.44                     & 1.42                                                & 1.53                     & 1.49                          \\ \midrule
				\multicolumn{1}{c|}{\begin{tabular}[c]{@{}c@{}}Rank\\ Sum\end{tabular}} &        22                &           \textbf{15}               &                            28                         &                        45 & 43                             &          20              &          27                &            \textbf{15}            &       23                   &                            48                         &                      62 & 62                              \\ \bottomrule
			\end{tabular}
		}
	\end{table*}

	\paragraph{Privacy Utility Trade-off} We analyze how privacy budget $\varepsilon$ influences the prediction quality under different numbers of public features.
	For $1\leq s^*  \leq 4$, we consider $n=10000$ and vary $\varepsilon$. 
	The results are displayed in Figure \ref{fig:privacyutility}. 
	When $\varepsilon$ increases, all MSEs decrease, while the performance steadily improves as the number of public features increases. 
	This observation is compatible with both Theorem \ref{thm:utility} and the notion that semi-feature LDP serves as the intermediate stage of LDP and label LDP.

	\paragraph{Consistency} We analyze the consistency of \eqref{equ:ourestimator} as $n$ grows. 
	We take $\varepsilon = 4$.
	The results are displayed in Figure \ref{fig:samplesize}, showing that the MSE is diminishing as $n$ grows, while the rate is faster for small $s$.

	\paragraph{Parameter Analysis}
	We conduct experiments to investigate the influence of key parameters $p$ and $t$. 
	We consider $n=10000$, $\varepsilon = 4$, and $s^* = 2$. 
	As shown in Figure \ref{fig:parameter}, for each $t$, as $p$ increases, MSE exhibits a U-shaped curve. 
	Also, the best performance is achieved with $t=2$, indicating such a phenomenon is also true for $t$. 
	This further emphasizes the importance of choosing correct hyperparameters as outlined in Theorem \ref{thm:utility}.
	Moreover, the depth $p$ at which the test error is minimized increases as $t$ increases. 
	This is compatible with theory since $t \asymp 2^{p / (d-s^* )}$.

	\paragraph{Selection of $s$ under Different Behaviors of $W$}
	We investigate how the behavior of $W$ affects the selection of $s$. 
	Specifically, we use $W_{\gamma}$ parameterized by $\gamma$, where $W_{\gamma i}^\ell = 1$ if $i\leq n / \ell^{\gamma}$. 
	A smaller $\gamma$ implies a thicker tail of $W$.
	We compare the performance of \texttt{HistOfTree} with different choices of $s$, as well as \texttt{AdHistOfTree} whose $s$ is selected automatically. 
	In Figure \ref{fig:selects}, the observations are two-fold. 
	First, MSE decreases w.r.t. $\gamma$, which aligns with Theorem \ref{thm:utility} that a thicker tail of $W$ leads to a larger upper bound. 
	On the other hand, one $s$ does not consistently outperform another under different thicknesses of the tail, which necessitates the discussions about $W$ in Section \ref{sec:utilityguarantee} and \ref{sec:selectionofs}. 
	The best $s$ lies between 1 and 2, while \texttt{AdHistOfTree} always achieves approximately the best performance. 
	This illustrates the effectiveness of the selection rule in \eqref{equ:obejectforparametertuning}.

	\subsection{Real Data}\label{sec:realdata}

	We conduct experiments on 10 real-world datasets that are available online.
	These datasets cover a wide range of sample sizes and feature dimensions commonly encountered in real-world scenarios.
	Some of these datasets contain both sensitive and insensitive features, making them suitable for our case. 
	Privacy preferences follow two paradigms.
	For both paradigms, we rank the features subjectively according to their level of sensitivity, from sensitive to non-sensitive. 
	In the aligned case, we select out the first $s^*  = \lceil\log \sqrt{d} \rceil$ features.
	In the personalized case, we set $W_i^{\ell} = 1$ if $i \leq n / 10^{\lfloor \ell / s^* \rfloor}$ and 0 otherwise. 
	A summary of key information for these datasets and pre-processing details is in Appendix \ref{app:datasets}.

	We first compute the mean squared error over 50 random train-test splits for $\varepsilon = 1,2,4$.
	To standardize the scale across datasets, we report the MSE ratio relative to non-private fitting a decision tree over the whole training samples.
	The results for $\varepsilon =2 $ are displayed in Table \ref{tab:realdataepsilon2}, while the others are presented in Appendix \ref{app:additionalresults}. 
	For all privacy budgets, the proposed methods significantly outperform competitors in terms of both average performance (rank sum) and the number of best results achieved.
	The data-driven selected parameters achieve comparable performance to the best parameter, indicating the effectiveness of the selection rule. 
	In some cases, \texttt{AdHistOfTree} performs slightly better, largely due to the incompleteness of parameter grids of \texttt{HistOfTree}.
	Moreover, we observe that the CART rule outperforms ME in the aligned setting, while the opposite holds in the personalized setting. 
	This is attributed to the reliance on sufficient public features of CART.

	\section{Limitations}\label{app:limitation}

	The study is a first step towards locally private learning with public features.
	There are several limitations that can serve as a guide for future work.
	The current lower bound can be further explored to see whether Theorem \ref{thm:utility} is tight. 
	Under the current scheme, we shall have $\sum_{i=1}^n \mathcal{E}^{\frac{2 \alpha+d+\sum_{j=1}^d W_i^j}{2 \alpha}} \geq \frac{1}{\left(e^{\varepsilon}-1\right)^2}$ for lower bound, where $\mathcal{E}$ is the excess risk. For the upper bound, we have $\sum_{i=1}^n \mathcal{E}^{-\frac{2 \alpha+d+\sum_{j=1}^d W_i^j}{2 \alpha}} \geq n^2 \varepsilon^2$. They are matched similarly as in the article if $\sum_{j=1}^d W_i^j$ are all equal, which includes the aligned case as a special case but is far more restrictive than the personalized case.

	During the rebuttal, one of the reviewers raised a concern that the privacy preferences used in the real-data experiments are subjective rather than derived from actual users. We acknowledge this limitation. Real-world data containing genuine privacy preferences is extremely rare and typically only obtainable through online testing. Such data must include both the value of a sensitive feature and the user’s preference for privatizing it. However, (1) collecting such data is uncommon, and (2) missing values may arise when users choose not to disclose sensitive features.
	Following the reviewer’s suggestion, we attempted to collect data using a public survey service. However, the collected data proved to be too noisy, and further data collection is not feasible due to limited funding.
	As a result, no additional experiments or datasets are included as contributions to this paper.
	
	\section*{Acknowledgments}
	
	Hanfang Yang is the corresponding author. 
	The authors would like to thank the reviewers for their constructive comments, which led to a significant improvement in this work. 
	The research is supported by the Special Funds of the National Natural Science Foundation of China (Grant No. 72342010). 
	Yuheng Ma is supported by the Outstanding Innovative Talents Cultivation Funded Programs 2024 of Renmin University of China.
	This research is also supported by Public Computing Cloud, Renmin University of China.

	\bibliographystyle{plainnat}
	\bibliography{references}

\begin{thebibliography}{65}
\providecommand{\natexlab}[1]{#1}
\providecommand{\url}[1]{\texttt{#1}}
\expandafter\ifx\csname urlstyle\endcsname\relax
  \providecommand{\doi}[1]{doi: #1}\else
  \providecommand{\doi}{doi: \begingroup \urlstyle{rm}\Url}\fi

\bibitem[Aliakbarpour et~al.(2024)Aliakbarpour, Chaudhuri, Courtade, Fallah,
  and Jordan]{aliakbarpour2024enhancing}
Maryam Aliakbarpour, Syomantak Chaudhuri, Thomas~A Courtade, Alireza Fallah,
  and Michael~I Jordan.
\newblock Enhancing feature-specific data protection via bayesian coordinate
  differential privacy.
\newblock \emph{arXiv preprint arXiv:2410.18404}, 2024.

\bibitem[Amorino and Gloter(2023)]{amorino2023minimax}
Chiara Amorino and Arnaud Gloter.
\newblock Minimax rate for multivariate data under componentwise local
  differential privacy constraints.
\newblock \emph{arXiv preprint arXiv:2305.10416}, 2023.

\bibitem[Badanidiyuru~Varadaraja et~al.(2024)Badanidiyuru~Varadaraja, Ghazi,
  Kamath, Kumar, Leeman, Manurangsi, Varadarajan, and
  Zhang]{badanidiyuru2024optimal}
Ashwinkumar Badanidiyuru~Varadaraja, Badih Ghazi, Pritish Kamath, Ravi Kumar,
  Ethan Leeman, Pasin Manurangsi, Avinash~V Varadarajan, and Chiyuan Zhang.
\newblock Optimal unbiased randomizers for regression with label differential
  privacy.
\newblock \emph{Advances in Neural Information Processing Systems}, 36, 2024.

\bibitem[Berrett et~al.(2021)Berrett, Gy{\"o}rfi, and
  Walk]{berrett2021strongly}
Thomas~B Berrett, L{\'a}szl{\'o} Gy{\"o}rfi, and Harro Walk.
\newblock Strongly universally consistent nonparametric regression and
  classification with privatised data.
\newblock \emph{Electronic Journal of Statistics}, 15:\penalty0 2430--2453,
  2021.

\bibitem[Breiman(1984)]{breiman1984classification}
L~Breiman.
\newblock Classification and regression trees.
\newblock \emph{The Wadsworth \& Brooks/Cole}, 1984.

\bibitem[Busa-Fekete et~al.(2021)Busa-Fekete, Syed, Vassilvitskii,
  et~al.]{busa2021population}
Robert~Istvan Busa-Fekete, Umar Syed, Sergei Vassilvitskii, et~al.
\newblock Population level privacy leakage in binary classification wtih label
  noise.
\newblock In \emph{NeurIPS 2021 Workshop Privacy in Machine Learning}, 2021.

\bibitem[Cai et~al.(2023)Cai, Ma, Dong, and Yang]{cai2023extrapolated}
Yuchao Cai, Yuheng Ma, Yiwei Dong, and Hanfang Yang.
\newblock Extrapolated random tree for regression.
\newblock In \emph{Proceedings of the 40th International Conference on Machine
  Learning}, pages 3442--3468. PMLR, 2023.

\bibitem[Chaudhuri and Hsu(2011)]{chaudhuri2011sample}
Kamalika Chaudhuri and Daniel Hsu.
\newblock Sample complexity bounds for differentially private learning.
\newblock In \emph{Proceedings of the 24th Annual Conference on Learning
  Theory}, pages 155--186. JMLR Workshop and Conference Proceedings, 2011.

\bibitem[Chua et~al.(2024)Chua, Cui, Ghazi, Harrison, Kamath, Krichene, Kumar,
  Manurangsi, Narra, Sinha, et~al.]{chua2024training}
Lynn Chua, Qiliang Cui, Badih Ghazi, Charlie Harrison, Pritish Kamath, Walid
  Krichene, Ravi Kumar, Pasin Manurangsi, Krishna~Giri Narra, Amer Sinha,
  et~al.
\newblock Training differentially private ad prediction models with
  semi-sensitive features.
\newblock \emph{arXiv preprint arXiv:2401.15246}, 2024.

\bibitem[Cortez et~al.(2009)Cortez, Cerdeira, Almeida, Matos, and
  Reis]{cortez2009modeling}
Paulo Cortez, Ant{\'o}nio Cerdeira, Fernando Almeida, Telmo Matos, and Jos{\'e}
  Reis.
\newblock Modeling wine preferences by data mining from physicochemical
  properties.
\newblock \emph{Decision support systems}, 47\penalty0 (4):\penalty0 547--553,
  2009.

\bibitem[Cummings and Desai(2018)]{cummings2018role}
Rachel Cummings and Deven Desai.
\newblock The role of differential privacy in gdpr compliance.
\newblock In \emph{Proceedings of the Conference on Fairness, Accountability,
  and Transparency}, page~20, 2018.

\bibitem[Cunningham et~al.(2022)Cunningham, Klemmer, Wen, and
  Ferhatosmanoglu]{cunningham2022geopointgan}
Teddy Cunningham, Konstantin Klemmer, Hongkai Wen, and Hakan Ferhatosmanoglu.
\newblock Geopointgan: Synthetic spatial data with local label differential
  privacy.
\newblock \emph{arXiv preprint arXiv:2205.08886}, 2022.

\bibitem[Curmei et~al.(2023)Curmei, Krichene, Zhang, and
  Sundararajan]{curmei2023private}
Mihaela Curmei, Walid Krichene, Li~Zhang, and Mukund Sundararajan.
\newblock Private matrix factorization with public item features.
\newblock In \emph{Proceedings of the 17th ACM Conference on Recommender
  Systems}, pages 805--812, 2023.

\bibitem[Dua and Graff(2017)]{Dua:2019}
Dheeru Dua and Casey Graff.
\newblock {UCI} machine learning repository, 2017.
\newblock URL \url{http://archive.ics.uci.edu/ml}.

\bibitem[Duchi et~al.(2018)Duchi, Jordan, and Wainwright]{duchi2018minimax}
John Duchi, Michael Jordan, and Martin Wainwright.
\newblock Minimax optimal procedures for locally private estimation.
\newblock \emph{Journal of the American Statistical Association}, 113\penalty0
  (521):\penalty0 182--201, 2018.

\bibitem[Dwork et~al.(2006)Dwork, McSherry, Nissim, and
  Smith]{dwork2006calibrating}
Cynthia Dwork, Frank McSherry, Kobbi Nissim, and Adam Smith.
\newblock Calibrating noise to sensitivity in private data analysis.
\newblock In \emph{Theory of cryptography conference}, pages 265--284.
  Springer, 2006.

\bibitem[Erlingsson et~al.(2014)Erlingsson, Pihur, and
  Korolova]{erlingsson2014rappor}
{\'U}lfar Erlingsson, Vasyl Pihur, and Aleksandra Korolova.
\newblock Rappor: Randomized aggregatable privacy-preserving ordinal response.
\newblock In \emph{Proceedings of the 2014 ACM SIGSAC conference on computer
  and communications security}, pages 1054--1067, 2014.

\bibitem[{European Parliament} and {Council of the European
  Union}()]{EuropeanParliament2016a}
{European Parliament} and {Council of the European Union}.
\newblock Regulation ({EU}) 2016/679 of the {European} {Parliament} and of the
  {Council}.
\newblock URL \url{https://data.europa.eu/eli/reg/2016/679/oj}.

\bibitem[Fuentes et~al.(2024)Fuentes, Mullins, McKenna, Miklau, and
  Sheldon]{fuentes2024joint}
Miguel Fuentes, Brett~C Mullins, Ryan McKenna, Gerome Miklau, and Daniel
  Sheldon.
\newblock Joint selection: Adaptively incorporating public information for
  private synthetic data.
\newblock In \emph{International Conference on Artificial Intelligence and
  Statistics}, pages 2404--2412. PMLR, 2024.

\bibitem[Ghazi et~al.(2021)Ghazi, Golowich, Kumar, Manurangsi, and
  Zhang]{ghazi2021deep}
Badih Ghazi, Noah Golowich, Ravi Kumar, Pasin Manurangsi, and Chiyuan Zhang.
\newblock Deep learning with label differential privacy.
\newblock \emph{Advances in neural information processing systems},
  34:\penalty0 27131--27145, 2021.

\bibitem[Ghazi et~al.(2022)Ghazi, Kamath, Kumar, Leeman, Manurangsi,
  Varadarajan, and Zhang]{ghazi2022regression}
Badih Ghazi, Pritish Kamath, Ravi Kumar, Ethan Leeman, Pasin Manurangsi,
  Avinash Varadarajan, and Chiyuan Zhang.
\newblock Regression with label differential privacy.
\newblock In \emph{The Eleventh International Conference on Learning
  Representations}, 2022.

\bibitem[Gu et~al.(2023)Gu, Kamath, and Wu]{gu2023choosing}
Xin Gu, Gautam Kamath, and Zhiwei~Steven Wu.
\newblock Choosing public datasets for private machine learning via gradient
  subspace distance.
\newblock \emph{arXiv preprint arXiv:2303.01256}, 2023.

\bibitem[Gy{\"o}rfi and Kroll(2022)]{gyorfi2022rate}
L{\'a}szl{\'o} Gy{\"o}rfi and Martin Kroll.
\newblock On rate optimal private regression under local differential privacy.
\newblock \emph{arXiv preprint arXiv:2206.00114}, 2022.

\bibitem[Inc.(2017)]{apple2017differential}
Apple Inc.
\newblock Differential privacy technical overview.
\newblock Technical Report Apple Inc., 2017.
\newblock URL
  \url{https://www.apple.com/privacy/docs/Differential_Privacy_Overview.pdf}.

\bibitem[Kairouz et~al.(2014)Kairouz, Oh, and Viswanath]{kairouz2014extremal}
Peter Kairouz, Sewoong Oh, and Pramod Viswanath.
\newblock Extremal mechanisms for local differential privacy.
\newblock \emph{Advances in neural information processing systems}, 27, 2014.

\bibitem[Kairouz et~al.(2016)Kairouz, Bonawitz, and
  Ramage]{kairouz2016discrete}
Peter Kairouz, Keith Bonawitz, and Daniel Ramage.
\newblock Discrete distribution estimation under local privacy.
\newblock In \emph{International Conference on Machine Learning}, pages
  2436--2444. PMLR, 2016.

\bibitem[Kosorok(2008)]{Kosorok2008introduction}
Michael~R. Kosorok.
\newblock \emph{Introduction to Empirical Processes and Semiparametric
  Inference}.
\newblock Springer Series in Statistics. Springer, New York, 2008.

\bibitem[Krichene et~al.(2024)Krichene, Mayoraz, Rendle, Song, Thakurta, and
  Zhang]{krichene2024private}
Walid Krichene, Nicolas~E Mayoraz, Steffen Rendle, Shuang Song, Abhradeep
  Thakurta, and Li~Zhang.
\newblock Private learning with public features.
\newblock In \emph{International Conference on Artificial Intelligence and
  Statistics}, pages 4150--4158. PMLR, 2024.

\bibitem[Li et~al.(2022)Li, Yan, Cheng, Sun, and Li]{li2022protecting}
Xiaoguang Li, Haonan Yan, Zelei Cheng, Wenhai Sun, and Hui Li.
\newblock Protecting regression models with personalized local differential
  privacy.
\newblock \emph{IEEE Transactions on Dependable and Secure Computing},
  20\penalty0 (2):\penalty0 960--974, 2022.

\bibitem[Liu et~al.(2021{\natexlab{a}})Liu, Vietri, Steinke, Ullman, and
  Wu]{liu2021leveraging}
Terrance Liu, Giuseppe Vietri, Thomas Steinke, Jonathan Ullman, and Steven Wu.
\newblock Leveraging public data for practical private query release.
\newblock In \emph{International Conference on Machine Learning}, pages
  6968--6977. PMLR, 2021{\natexlab{a}}.

\bibitem[Liu et~al.(2021{\natexlab{b}})Liu, Vietri, and Wu]{liu2021iterative}
Terrance Liu, Giuseppe Vietri, and Steven~Z Wu.
\newblock Iterative methods for private synthetic data: Unifying framework and
  new methods.
\newblock \emph{Advances in Neural Information Processing Systems},
  34:\penalty0 690--702, 2021{\natexlab{b}}.

\bibitem[Liu and Yeh(2017)]{liu2017using}
Yi-Cheng Liu and I-Cheng Yeh.
\newblock Using mixture design and neural networks to build stock selection
  decision support systems.
\newblock \emph{Neural Computing and Applications}, 28:\penalty0 521--535,
  2017.

\bibitem[Ma and Yang(2024)]{ma2024optimal}
Yuheng Ma and Hanfang Yang.
\newblock Optimal locally private nonparametric classification with public
  data.
\newblock \emph{Journal of Machine Learning Research}, 25\penalty0
  (167):\penalty0 1--62, 2024.

\bibitem[Ma et~al.(2023)Ma, Zhang, Cai, and Hanfang]{ma2023decision}
Yuheng Ma, Han Zhang, Yuchao Cai, and Yang Hanfang.
\newblock Decision tree for locally private estimation with public data.
\newblock \emph{Advances in Neural Information Processing Systems}, 2023.

\bibitem[Ma et~al.(2024)Ma, Jia, and Yang]{ma2024better}
Yuheng Ma, Ke~Jia, and Hanfang Yang.
\newblock Better locally private sparse estimation given multiple samples per
  user.
\newblock In \emph{Proceedings of the 41st International Conference on Machine
  Learning}, volume 235, pages 33746--33776, 2024.

\bibitem[Malek~Esmaeili et~al.(2021)Malek~Esmaeili, Mironov, Prasad, Shilov,
  and Tramer]{malek2021antipodes}
Mani Malek~Esmaeili, Ilya Mironov, Karthik Prasad, Igor Shilov, and Florian
  Tramer.
\newblock Antipodes of label differential privacy: Pate and alibi.
\newblock \emph{Advances in Neural Information Processing Systems},
  34:\penalty0 6934--6945, 2021.

\bibitem[Nash et~al.(1994)Nash, Sellers, Talbot, Cawthorn, and
  Ford]{nash1994population}
Warwick~J Nash, Tracy~L Sellers, Simon~R Talbot, Andrew~J Cawthorn, and Wes~B
  Ford.
\newblock The population biology of abalone (haliotis species) in tasmania. i.
  blacklip abalone (h. rubra) from the north coast and islands of bass strait.
\newblock \emph{Sea Fisheries Division, Technical Report}, 48:\penalty0 p411,
  1994.

\bibitem[Nasr et~al.(2023)Nasr, Mahloujifar, Tang, Mittal, and
  Houmansadr]{nasr2023effectively}
Milad Nasr, Saeed Mahloujifar, Xinyu Tang, Prateek Mittal, and Amir Houmansadr.
\newblock Effectively using public data in privacy preserving machine learning.
\newblock In \emph{Proceedings of the 40th International Conference on Machine
  Learning}, volume 202. PMLR, 2023.

\bibitem[Nouwens et~al.(2020)Nouwens, Liccardi, Veale, Karger, and
  Kagal]{nouwens2020dark}
Midas Nouwens, Ilaria Liccardi, Michael Veale, David Karger, and Lalana Kagal.
\newblock Dark patterns after the gdpr: Scraping consent pop-ups and
  demonstrating their influence.
\newblock In \emph{Proceedings of the 2020 CHI conference on human factors in
  computing systems}, pages 1--13, 2020.

\bibitem[Papernot et~al.(2017)Papernot, Abadi, Erlingsson, Goodfellow, and
  Talwar]{papernotsemi}
Nicolas Papernot, Mart{\'\i}n Abadi, {\'U}lfar Erlingsson, Ian Goodfellow, and
  Kunal Talwar.
\newblock Semi-supervised knowledge transfer for deep learning from private
  training data.
\newblock In \emph{International Conference on Learning Representations}, 2017.

\bibitem[Papernot et~al.(2018)Papernot, Song, Mironov, Raghunathan, Talwar, and
  Erlingsson]{papernot2018scalable}
Nicolas Papernot, Shuang Song, Ilya Mironov, Ananth Raghunathan, Kunal Talwar,
  and Ulfar Erlingsson.
\newblock Scalable private learning with pate.
\newblock In \emph{International Conference on Learning Representations}, 2018.

\bibitem[Pedregosa et~al.(2011)Pedregosa, Varoquaux, Gramfort, Michel, Thirion,
  Grisel, Blondel, Prettenhofer, Weiss, Dubourg, Vanderplas, Passos,
  Cournapeau, Brucher, Perrot, and Duchesnay]{scikit-learn}
F.~Pedregosa, G.~Varoquaux, A.~Gramfort, V.~Michel, B.~Thirion, O.~Grisel,
  M.~Blondel, P.~Prettenhofer, R.~Weiss, V.~Dubourg, J.~Vanderplas, A.~Passos,
  D.~Cournapeau, M.~Brucher, M.~Perrot, and E.~Duchesnay.
\newblock Scikit-learn: Machine learning in {P}ython.
\newblock \emph{Journal of Machine Learning Research}, 12:\penalty0 2825--2830,
  2011.

\bibitem[Shen et~al.(2023)Shen, Krishnaswamy, Kulkarni, and
  Munagala]{shen2023classification}
Zeyu Shen, Anilesh Krishnaswamy, Janardhan Kulkarni, and Kamesh Munagala.
\newblock Classification with partially private features.
\newblock \emph{arXiv preprint arXiv:2312.07583}, 2023.

\bibitem[Song et~al.(2019)Song, Luo, Wang, and Li]{song2019multiple}
Haina Song, Tao Luo, Xun Wang, and Jianfeng Li.
\newblock Multiple sensitive values-oriented personalized privacy preservation
  based on randomized response.
\newblock \emph{IEEE transactions on information forensics and security},
  15:\penalty0 2209--2224, 2019.

\bibitem[Sun et~al.(2014)Sun, Fu, Zhou, Gao, et~al.]{sun2014personalized}
Chongjing Sun, Yan Fu, Junlin Zhou, Hui Gao, et~al.
\newblock Personalized privacy-preserving frequent itemset mining using
  randomized response.
\newblock \emph{The Scientific World Journal}, 2014, 2014.

\bibitem[Tsybakov(2009)]{tsybakov2008introduction}
Alexandre~B. Tsybakov.
\newblock \emph{Introduction to Nonparametric Estimation}.
\newblock Springer Series in Statistics. Springer, New York, 2009.
\newblock ISBN 978-0-387-79051-0.
\newblock \doi{10.1007/b13794}.
\newblock URL
  \url{http://hfhhc391f4815d8064db7sqbf60x0xf60c6kpo.faxb.libproxy.ruc.edu.cn/10.1007/b13794}.
\newblock Revised and extended from the 2004 French original, Translated by
  Vladimir Zaiats.

\bibitem[van~der Vaart and Wellner(1996)]{vandervaart1996weak}
Aad~W. van~der Vaart and Jon~A. Wellner.
\newblock \emph{Weak Convergence and Empirical Processes}.
\newblock Springer Series in Statistics. Springer-Verlag, New York, 1996.

\bibitem[Wainwright(2019)]{wainwright2019high}
Martin~J Wainwright.
\newblock \emph{High-dimensional statistics: A non-asymptotic viewpoint},
  volume~48.
\newblock Cambridge University Press, 2019.

\bibitem[Wang et~al.(2024)Wang, Zhu, Su, and Wang]{wang2024neural}
Chendi Wang, Yuqing Zhu, Weijie~J Su, and Yu-Xiang Wang.
\newblock Neural collapse meets differential privacy: Curious behaviors of
  noisygd with near-perfect representation learning.
\newblock In \emph{Forty-first International Conference on Machine Learning},
  2024.

\bibitem[Wang and Xu(2019)]{wang2019sparse}
Di~Wang and Jinhui Xu.
\newblock On sparse linear regression in the local differential privacy model.
\newblock In \emph{International Conference on Machine Learning}, pages
  6628--6637. PMLR, 2019.

\bibitem[Wang et~al.(2015)Wang, Huang, Tian, Yang, Xu, and
  Guo]{wang2015personalized}
Shaowei Wang, Liusheng Huang, Miaomiao Tian, Wei Yang, Hongli Xu, and Hansong
  Guo.
\newblock Personalized privacy-preserving data aggregation for histogram
  estimation.
\newblock In \emph{2015 IEEE Global Communications Conference (GLOBECOM)},
  pages 1--6. IEEE, 2015.

\bibitem[Wang et~al.(2017)Wang, Blocki, Li, and Jha]{wang2017locally}
Tianhao Wang, Jeremiah Blocki, Ninghui Li, and Somesh Jha.
\newblock Locally differentially private protocols for frequency estimation.
\newblock In \emph{26th USENIX Security Symposium (USENIX Security 17)}, pages
  729--745, 2017.

\bibitem[Warner(1965)]{warner1965randomized}
Stanley~L Warner.
\newblock Randomized response: A survey technique for eliminating evasive
  answer bias.
\newblock \emph{Journal of the American Statistical Association}, 60\penalty0
  (309):\penalty0 63--69, 1965.

\bibitem[Wilcoxon(1992)]{wilcoxon1992individual}
Frank Wilcoxon.
\newblock Individual comparisons by ranking methods.
\newblock In \emph{Breakthroughs in statistics}, pages 196--202. Springer,
  1992.

\bibitem[Xiong et~al.(2020)Xiong, Liu, Li, Cai, and
  Niu]{xiong2020comprehensive}
Xingxing Xiong, Shubo Liu, Dan Li, Zhaohui Cai, and Xiaoguang Niu.
\newblock A comprehensive survey on local differential privacy.
\newblock \emph{Security and Communication Networks}, 2020:\penalty0 1--29,
  2020.

\bibitem[Xu et~al.(2023)Xu, Wang, Sun, and Cheng]{xu2023binary}
Shirong Xu, Chendi Wang, Will~Wei Sun, and Guang Cheng.
\newblock Binary classification under local label differential privacy using
  randomized response mechanisms.
\newblock \emph{Transactions on Machine Learning Research}, 2023.

\bibitem[Yang et~al.(2021)Yang, Wang, and Wang]{yang2021federated}
Ge~Yang, Shaowei Wang, and Haijie Wang.
\newblock Federated learning with personalized local differential privacy.
\newblock In \emph{2021 IEEE 6th International Conference on Computer and
  Communication Systems (ICCCS)}, pages 484--489. IEEE, 2021.

\bibitem[Yu et~al.(2021)Yu, Zhang, Chen, Yin, and Liu]{yu2021large}
Da~Yu, Huishuai Zhang, Wei Chen, Jian Yin, and Tie-Yan Liu.
\newblock Large scale private learning via low-rank reparametrization.
\newblock In \emph{International Conference on Machine Learning}, pages
  12208--12218. PMLR, 2021.

\bibitem[Yu et~al.(2022)Yu, Naik, Backurs, Gopi, Inan, Kamath, Kulkarni, Lee,
  Manoel, Wutschitz, Yekhanin, and Zhang]{yu2022differentially}
Da~Yu, Saurabh Naik, Arturs Backurs, Sivakanth Gopi, Huseyin~A. Inan, Gautam
  Kamath, Janardhan Kulkarni, Yin~Tat Lee, Andre Manoel, Lukas Wutschitz,
  Sergey Yekhanin, and Huishuai Zhang.
\newblock Differentially private fine-tuning of language models.
\newblock In \emph{The Tenth International Conference on Learning
  Representations}, 2022.

\bibitem[Zhang et~al.(2024)Zhang, Zhang, Zhang, and Zhang]{zhang2024adaptive}
Dongyan Zhang, Lili Zhang, Zhiyong Zhang, and Zhongya Zhang.
\newblock Adaptive personalized randomized response method based on local
  differential privacy.
\newblock \emph{International Journal of Information Security and Privacy
  (IJISP)}, 18\penalty0 (1):\penalty0 1--19, 2024.

\bibitem[Zhang et~al.(2022)Zhang, Chen, Jia, and Zhai]{zhang2022support}
Guopeng Zhang, Xuebin Chen, Yuanli Jia, and Ran Zhai.
\newblock Support personalized weighted local differential privacy skyline
  query.
\newblock \emph{Security and Communication Networks}, 2022, 2022.

\bibitem[Zhao et~al.(2024{\natexlab{a}})Zhao, Shen, Fan, Li, Wu, Wu, and
  Liu]{zhao2024learning}
Puning Zhao, Li~Shen, Rongfei Fan, Qingming Li, Huiwen Wu, Jiafei Wu, and Zhe
  Liu.
\newblock Learning with user-level local differential privacy.
\newblock \emph{arXiv preprint arXiv:2405.17079}, 2024{\natexlab{a}}.

\bibitem[Zhao et~al.(2024{\natexlab{b}})Zhao, Wu, Wu, and
  Liu]{zhao2024theoretical}
Puning Zhao, Huiwen Wu, Jiafei Wu, and Zhe Liu.
\newblock On theoretical limits of learning with label differential privacy.
\newblock 2024{\natexlab{b}}.

\bibitem[Zhao et~al.(2025)Zhao, Wu, Liu, Shen, Zhang, Fan, Sun, and
  Li]{zhao2025enhancing}
Puning Zhao, Jiafei Wu, Zhe Liu, Li~Shen, Zhikun Zhang, Rongfei Fan, Le~Sun,
  and Qingming Li.
\newblock Enhancing learning with label differential privacy by vector
  approximation.
\newblock In \emph{The Thirteenth International Conference on Learning
  Representations}, 2025.
\newblock URL \url{https://openreview.net/forum?id=IwPXYk6BV9}.

\bibitem[Zhou et~al.(2014)Zhou, Claire, and King]{zhou2014predicting}
Fang Zhou, Q~Claire, and Ross~D King.
\newblock Predicting the geographical origin of music.
\newblock In \emph{2014 IEEE International Conference on Data Mining}, pages
  1115--1120. IEEE, 2014.

\end{thebibliography}

	\section*{Checklist}

	The checklist follows the references. For each question, choose your answer from the three possible options: Yes, No, Not Applicable.  You are encouraged to include a justification to your answer, either by referencing the appropriate section of your paper or providing a brief inline description (1-2 sentences). 
	Please do not modify the questions.  Note that the Checklist section does not count towards the page limit. Not including the checklist in the first submission won't result in desk rejection, although in such case we will ask you to upload it during the author response period and include it in camera ready (if accepted).
	
	\textbf{In your paper, please delete this instructions block and only keep the Checklist section heading above along with the questions/answers below.}

	\begin{enumerate}

		\item For all models and algorithms presented, check if you include:
		\begin{enumerate}
			\item A clear description of the mathematical setting, assumptions, algorithm, and/or model. [Yes]
			\item An analysis of the properties and complexity (time, space, sample size) of any algorithm. [Yes]
			\item (Optional) Anonymized source code, with specification of all dependencies, including external libraries. [No]
		\end{enumerate}

		\item For any theoretical claim, check if you include:
		\begin{enumerate}
			\item Statements of the full set of assumptions of all theoretical results. [Yes]
			\item Complete proofs of all theoretical results. [Yes]
			\item Clear explanations of any assumptions. [Yes]     
		\end{enumerate}

		\item For all figures and tables that present empirical results, check if you include:
		\begin{enumerate}
			\item The code, data, and instructions needed to reproduce the main experimental results (either in the supplemental material or as a URL). [Yes]
			\item All the training details (e.g., data splits, hyperparameters, how they were chosen). [Yes]
			\item A clear definition of the specific measure or statistics and error bars (e.g., with respect to the random seed after running experiments multiple times). [Yes]
			\item A description of the computing infrastructure used. (e.g., type of GPUs, internal cluster, or cloud provider). [Yes]
		\end{enumerate}
		
		\item If you are using existing assets (e.g., code, data, models) or curating/releasing new assets, check if you include:
		\begin{enumerate}
			\item Citations of the creator If your work uses existing assets. [Not Applicable]
			\item The license information of the assets, if applicable. [Not Applicable]
			\item New assets either in the supplemental material or as a URL, if applicable. [Not Applicable]
			\item Information about consent from data providers/curators. [Not Applicable]
			\item Discussion of sensible content if applicable, e.g., personally identifiable information or offensive content. [Not Applicable]
		\end{enumerate}
		
		\item If you used crowdsourcing or conducted research with human subjects, check if you include:
		\begin{enumerate}
			\item The full text of instructions given to participants and screenshots. [Not Applicable]
			\item Descriptions of potential participant risks, with links to Institutional Review Board (IRB) approvals if applicable. [Not Applicable]
			\item The estimated hourly wage paid to participants and the total amount spent on participant compensation. [Not Applicable]
		\end{enumerate}
		
	\end{enumerate}

	\newpage

	\appendix

	\onecolumn

	\section*{Appendix}

	In this appendix, we provide the detailed results for methodology (Appendix \ref{app:methodology}), the proofs (Appendix \ref{sec:proofs}), and experiments (Appendix \ref{app:experiments}).

	\section{Related Methodology of Locally Differentially Private Decision Tree}\label{app:methodology}

	\subsection{Generalized Randomized Response Mechanism}\label{app:generelizedrandomresponse}
	
	In this section, we introduce the generalized randomized response mechanism \citep{warner1965randomized, kairouz2016discrete, ghazi2022regression} (or direct encoding \citep{wang2017locally}). 
	Specifically, for the aligned privacy preference case, the mechanism calculates \eqref{equ:privavyprocedureX} by
	\begin{align}\label{equ:privavyprocedureXgeneralized}
		\mathrm{Pr}\left[U_i = j\right] = \begin{cases}  \frac{e^{\varepsilon / 2}}{e^{\varepsilon / 2} + |\pi^{\text{priv}}| - 1}, & \text{ if } \eins_{A_j}(X_i^{\text{priv}}) \\  \frac{1}{e^{\varepsilon / 2} + |\pi^{\text{priv}}| - 1}, & \text{ otherwise. } \end{cases}
	\end{align}
	and uses 
	$\tilde{U}_{i}^j =  \frac{e^{\varepsilon / 2} + |\pi^{\text{priv}}| - 1}{e^{\varepsilon / 2} - 1} \left( \eins_{U_i = j} - \frac{1}{e^{\varepsilon / 2} + |\pi^{\text{priv}}| - 1}  \right)$ will serve as an estimator of $\eins_{A_j}(X_i^{\text{priv}})$, since $\mathbb{E}_{\mathrm{P}, \mathrm{R}}\left[\tilde{U}_{i}^j\right] =\mathbb{E}_{\mathrm{P}}\left[\eins_{A_j}(X_i^{\text{priv}})\right] $. 
	The estimator \eqref{equ:ourestimator} remains identical. 
	For the personalized case, the mechanism computes \eqref{equ:privavyprocedurepersonalize} by 
	\begin{align}\label{equ:privavyprocedurepersonalizegeneralized}
		\tilde{Y}_i = Y_i +  \frac{4M}{\varepsilon} \xi_i, \quad 	\mathrm{Pr}\left[V_i = j\right]  = \begin{cases} 0 & \text{ if } A\times B \notin \mathcal{V}_i,  \\ 
			\frac{e^{\varepsilon / 2}}{e^{\varepsilon / 2} + |\mathcal{V}_i| - 1}
			& \text { elif } \eins_{A\times B}(X_i),   \\ \frac{1}{e^{\varepsilon / 2} + |\mathcal{V}_i| - 1}, & \text { otherwise. } \end{cases}
	\end{align}
	To estimate indicator function $\eins_{A\times B} (X_i)$, we use 
	\begin{align*}
		\tilde{V}_{i}^j =  \frac{e^{\varepsilon / 2} + |\mathcal{V}_i| - 1}{e^{\varepsilon / 2} - 1} \left( \eins_{V_i = j} - \frac{1}{e^{\varepsilon / 2} + |\mathcal{V}_i| - 1}  \right) \cdot \eins_{A\times B \in \mathcal{V}_i}.
	\end{align*}
	In this case, the indexes of potential grids are estimated analogously to $\tilde{U}_i^j$, while the remaining grids are free of privacy concern and zeroed out. 
	The privacy guarantee of \eqref{equ:privavyprocedurepersonalizegeneralized} is guaranteed by the following proposition. 
	
	\begin{proposition}\label{prop:privacygeneralized}
		Let $\pi = \{ A\times B\mid A\in\pi^{\text{priv}}, B\in \pi^{\text{pub}}\}$ be any partition of $\mathcal{X}$.
		Then the privacy mechanism \eqref{equ:privavyprocedurepersonalizegeneralized} is $\varepsilon$-semi-feature LDP. 
	\end{proposition}
	
	The reason we introduce this mechanism is due to the fact that when $|\mathcal{V}_i|$ is small, it has been shown to be quite efficient \citep{wang2017locally, kairouz2016discrete, xiong2020comprehensive}. In experiments where sample sizes $n\sim 1e3$ and $|\mathcal{V}_i| \sim 1e1$, we find that the generalized randomized response mechanism performs slightly better than the randomized responses introduced in the main text. However, the variance of such a mechanism grows linearly with $|\mathcal{V}_i|$, which can be troublesome in theoretical analysis where $|\mathcal{V}_i| $ is $n$ raised to some power. Consequently, in such cases, the convergence rate fails to attain optimality.

	\subsection{Complexity Analysis}\label{sec:complexity}

	We analyze the complexity of \texttt{HistOfTree} from perspectives of computation, memory, and communication.
	The overall communication rounds for personalized estimator \eqref{equ:ourestimatorpersonal} is two.
	Each user first sends its privatized label $\tilde{Y}_i$ and public features $X_i^{\text{pub}}$, which costs communication capacity of at most $d + 1$ real numbers.  
	Based on the created partition, it then sends the binary encoding $\tilde{V}_i^j$s, which are $t^s 2^{p}$ bools in total.

	The computation complexity of \texttt{HistOfTree} consists of two parts, creating the partition and computing the node values.
	The partition procedure takes $\mathcal{O}(p n d )$ time and the computation of \eqref{equ:ourestimatorpersonal} takes $\mathcal{O}(t^s2^p)$ time. 
	From the proof of Theorem \ref{thm:utility}, we know that $t^s2^{p} \lesssim n\varepsilon^2 $. 
	Thus the training stage complexity is at most $\mathcal{O}(n \varepsilon^2 + n d \log n\varepsilon^2)$, which is linear in $n$. 
	Since each prediction of the decision tree takes $\mathcal{O}(p)$ time, the test time for each test instance is $\mathcal{O} (\log n\varepsilon^2)$. 
	As for storage complexity, since \texttt{HistOfTree}  only requires storage of the tree structure and the node values, the space complexity is also $\mathcal{O}(t^s2^{p})$. 
	In short, \texttt{HistOfTree} is an efficient method in terms of computation, memory, and communication.

	\section{Proofs}\label{sec:proofs}
	
	\subsection{Proof of Proposition \ref{prop:privacy}}
	
	\begin{proof}[\textbf{Proof of Proposition  \ref{prop:privacy}}]
		Since our mechanism \eqref{equ:privavyprocedurepersonalize} does not rely on previous private information $Z_1,\cdots, Z_{i-1}$ when drawing $Z_i$, it is a non-interactive mechanism. 
		Each user only reveals two pieces of information that is private, i.e. $\tilde{Y}_i$ and $\tilde{V}_i$.
		Note that they are conditionally independent, thus for each conditional distribution $\mathrm{R}_i\left( \tilde{Y}_i , \tilde{V}_i,| X_i^{\text{priv}} = x_i^{\text{priv}},  X_i^{\text{pub}} = x_i^{\text{pub}}, Y_i = y\right)$, we can compute the density ratio as 
		\begin{align}\nonumber
			& \sup_{x_i^{\text{priv}} x_i^{\text{priv}\prime}, x_i^{\text{pub}},y,y'}\frac{\mathrm{R}_i\left( \tilde{Y}_i, \tilde{V}_i| X_i^{\text{priv}} = x_i^{\text{priv}},  X_i^{\text{pub}} = x_i^{\text{pub}}, Y_i = y \right)}{\mathrm{R}_i\left( \tilde{Y}_i, \tilde{V}_i| X_i^{\text{priv}} = x_i^{\text{priv}\prime},  X_i^{\text{pub}} = x_i^{\text{pub}}, Y_i = y' \right)} \\
			\leq  &  \sup_{x_i^{\text{priv}} x_i^{\text{priv}\prime}, x_i^{\text{pub}}} \frac{\mathrm{R}_i(\tilde{V}_i |X_i^{\text{priv}} = x_i^{\text{priv}},  X_i^{\text{pub}} = x_i^{\text{pub}}) }{\mathrm{R}_i(\tilde{V}_i |X_i^{\text{priv}} = x_i^{\text{priv}\prime},  X_i^{\text{pub}} = x_i^{\text{pub}}) } \cdot  \sup_{y,y'} \frac{\mathrm{R}_i(\tilde{Y}_i|Y_i = y)}{\mathrm{R}_i(\tilde{Y}_i|Y_i = y')}.\label{equ:densityratiodecomposition}
		\end{align}
		(\textit{i}) 
		The first part associates to the randomized response mechanism. 
		Since the conditional density is identical if $x^{\text{priv}}$ and $x^{\text{priv}\prime}$ belongs to a same $A_j$, we have 
		\begin{align*}
			\frac{\mathrm{R}_i(\tilde{V}_i |X_i^{\text{priv}} = x_i^{\text{priv}},  X_i^{\text{pub}} = x_i^{\text{pub}}) }{\mathrm{R}_i(\tilde{V}_i |X_i^{\text{priv}} = x_i^{\text{priv}\prime},  X_i^{\text{pub}} = x_i^{\text{pub}}) } =  \frac{\prod_j\mathrm{R}_i(\tilde{V}_i^j |X_i^{\text{priv}} = x_i^{\text{priv}},  X_i^{\text{pub}} = x_i^{\text{pub}})  }{\prod_j\mathrm{R}_i(\tilde{V}_i^j |X_i^{\text{priv}} = x_i^{\text{priv}\prime},  X_i^{\text{pub}} = x_i^{\text{pub}})  }.
		\end{align*}
		By definition, given different $x$, there are only two $\eins_{A\times B}(x)$ are different for all $A\times B$. 
		Without loss of generality, assume they differ on the first two elements. 
		Then we have
		\begin{align}\nonumber
			& \sup_{x_i^{\text{priv}} x_i^{\text{priv}\prime}, x_i^{\text{pub}}} \frac{\mathrm{R}_i(\tilde{V}_i |X_i^{\text{priv}} = x_i^{\text{priv}},  X_i^{\text{pub}} = x_i^{\text{pub}}) }{\mathrm{R}_i(\tilde{V}_i |X_i^{\text{priv}} = x_i^{\text{priv}\prime},  X_i^{\text{pub}} = x_i^{\text{pub}}) }  \\
			= & \sup_{x_i^{\text{priv}} x_i^{\text{priv}\prime}, x_i^{\text{pub}}}\frac{\prod_{j= 1,2} \mathrm{R}_i(\tilde{V}_i^j |X_i^{\text{priv}} = x_i^{\text{priv}},  X_i^{\text{pub}} = x_i^{\text{pub}}) }{\prod_{j=1,2} \mathrm{R}_i(\tilde{V}_i^j |X_i^{\text{priv}} = x_i^{\text{priv}\prime},  X_i^{\text{pub}} = x_i^{\text{pub}}) } \leq e^{\varepsilon/4 +\varepsilon/4} = e^{\varepsilon/2}. \label{equ:densityratioresultU}
		\end{align}
		(\textit{ii}) For the second part, there holds
		\begin{align}\nonumber
			\sup_{y,y'} \frac{\mathrm{R}_i(\tilde{Y}_i|Y_i = y)}{\mathrm{R}_i(\tilde{Y}_i|Y_i = y')} \leq \sup_{y,y'} \frac{d\mathrm{R}_i(\tilde{Y}_i|Y_i = y)}{d\mathrm{R}_i(\tilde{Y}_i|Y_i = y')}=  &\sup_{y,y'} \frac{\exp \left(-\frac{\varepsilon}{4M}|\tilde{Y}_i - y|\right)}{\exp \left(-\frac{\varepsilon}{4M}|\tilde{Y}_i - y'|\right)}
			\\
			\leq & \sup_{y,y'} \exp\left(\frac{\varepsilon}{4M}|y-y'|\right) \leq e^{\varepsilon/2}.\label{equ:densityratioresultY}
		\end{align}
		Bringing \eqref{equ:densityratioresultU} and \eqref{equ:densityratioresultY} into \eqref{equ:densityratiodecomposition} yields the desired conclusion. 
	\end{proof}

	\begin{proof}[\textbf{Proof of Proposition \ref{prop:privacygeneralized}}]
		We only need to modify (i) of the previous proof, which is now characterized by the generalized randomized response mechanism. 
		For each grid $A\times B$ with index $j$, we consider the three cases in \eqref{equ:privavyprocedurepersonalize}. 
		The first case is $A\times B\notin \mathcal{V}_i$. 
		The definition of $\mathcal{V}_i$ yields that, for any $X_i' = (x^{\text{priv}\prime}_i, x^{\text{pub}}_i)$, its potential grids $\mathcal{V}_i'$ is exactly the same as $\mathcal{V}_i$. 
		Thus, we have $\mathrm{R}_i(V_i = k |X_i^{\text{priv}} = x_i^{\text{priv}},  X_i^{\text{pub}} = x_i^{\text{pub}})  =\mathrm{R}_i(V_i = k  |X_i^{\text{priv}} = x_i^{\text{priv}\prime},  X_i^{\text{pub}} = x_i^{\text{pub}}) = 0$. 
		And thus the likelihood ratio is bounded by $e^{\varepsilon / 2}$ as we define $0/0=1$. 
		The second case is $A\times B \in \mathcal{V}_i$ and $X_i \in A\times B$. 
		In this case, any $X_i' = (x^{\text{priv}\prime}_i, x^{\text{pub}}_i)$ will have $A\times B \in \mathcal{V}_i'$. 
		As a result, 
		\begin{align*}
			\frac{\mathrm{R}_i(V_i = k|X_i^{\text{priv}} = x_i^{\text{priv}},  X_i^{\text{pub}} = x_i^{\text{pub}})  }{\mathrm{R}_i(V_i = k |X_i^{\text{priv}} = x_i^{\text{priv}\prime},  X_i^{\text{pub}} = x_i^{\text{pub}}) }  \leq \frac{\frac{e^{\varepsilon / 2}}{e^{\varepsilon / 2} + |\mathcal{V}_i| - 1}}{\frac{1}{e^{\varepsilon / 2} + |\mathcal{V}_i| - 1}} = e^{\varepsilon /2 }. 
		\end{align*}
		The third case is the symmetric situation of the second case. 
		Three cases together, we have
		\begin{align*}
			\frac{\mathrm{R}_i(V_i |X_i^{\text{priv}} = x_i^{\text{priv}},  X_i^{\text{pub}} = x_i^{\text{pub}}) }{\mathrm{R}_i(V_i |X_i^{\text{priv}} = x_i^{\text{priv}\prime},  X_i^{\text{pub}} = x_i^{\text{pub}}) }\leq e^{\varepsilon / 2}. 
		\end{align*}
	\end{proof}
	
	\subsection{Proof of Theorem \ref{thm:lowerbound}} \label{app:lowerbound}

	\begin{proof}[\textbf{Proof of Theorem \ref{thm:lowerbound}}]
		The lower bound contains two terms. 
		The second term $n^{-\frac{2\alpha}{2\alpha + d}}$ is the classical mini-max lower bound for non-private learners under Assumption \ref{asp:alphaholder}.
		See \citep{tsybakov2008introduction} for a comprehensive discussion. 
		In the following, we prove the conclusion for the first term. 
		The overall proof strategy mimics that of \citep{gyorfi2022rate}, while the utilization of information inequality is different. 
		We first construct a finite set of hypotheses distribution. 
		We then combine the information inequality under privacy constraint \citep{duchi2018minimax} and Assouad's Lemma \citep{tsybakov2008introduction}.

		Let $K_0: \mathbb{R} \rightarrow[0, \infty)$ be a $\alpha$-Holder continuous function with constant $1$ such that $|K_0(x_1) - K_0(x_2)| \leq \|x_1 - x_2\|_2^{\alpha}$. 
		Also, let $\operatorname{supp}\left(K_0\right) \subseteq[0,1]$. 
		For $x=\left(x_1, \ldots, x_d\right) \in[0,1]^d$, define the function $K:[0,1]^d \rightarrow \mathbb{R}$ via $K(x)=\min _{i=1, \ldots, d} K_0\left(x_i\right)$. 
		W.o.l.g., we let $ \int_{[0,1]^d} K(x') dx' = 1$ and $\|K\|_{\infty} = 2$. 
		We define the index set $\sigma = (\sigma^1,\cdots, \sigma^d) \in \{0,\cdots, k - 1\}^{d}$. 
		We write $\sigma = (\sigma_1, \sigma_2)$, where $\sigma_1\in \{0, \cdots, k-1\}^{s^*}$ is the first $s^*$ entries representing indexes on private dimensions, and $\sigma_2\in \{0, \cdots, k-1\}^{d-s^*}$ representing indexes on public dimensions. 
		Let $\theta^1  \in \{0,1\}^{k^{s^*}}$ be some one-hot encoding of $\sigma_1$ and $\theta^2\in\{0,1\}^{k^{d - s^*}}$ for $\sigma_2$ accordingly. 
		Denote $\theta = (\theta^1, \theta^2)$. 
		In this case, we have a bijection between a $\theta$ and a $\sigma$. 
		Define the function
		\begin{align*}
			K^{\theta}(x)= \frac{1}{k^{\alpha}} K (k x^1 - \sigma^1, \cdots, k x^d - \sigma^d)
		\end{align*}
		for integer $k$ and $x \in  \times_{i=1}^d\left[\sigma^i / k,\left(\sigma^i+1\right) / k\right]$. 
		We consider the following class of distributions $\mathrm{P}^{\theta}$ where $\mathrm{P}^{\theta }_X: [0,1]^d \rightarrow [0, \infty)$ is the uniform distribution, and $Y|X = x$ is the uniform distribution on $[ K^{\theta}(x) -1 / 2, K^{\theta}(x) + 1/2]$.

		We first verify that the constructed $\mathrm{P}^{\theta}$s satisfy Assumption \ref{asp:alphaholder}. 
		The marginal distribution is uniform and thus has bounded density. 
		The conditional expectation $K^{\theta}$ has 
		\begin{align*}
			|K^{\theta} (x_1) - K^{\theta} (x_2)| = \frac{1}{k^{\alpha}} |K(x_1) - K(x_2)|\leq  \frac{1}{k^{\alpha}} \| k x_1 - k x_2\|^{\alpha}_2  =  \| x_1 -  x_2\|^{\alpha}_2
		\end{align*}
		for all $\theta$.
		Thus, all $\mathrm{P}^{\theta}$ is readily checked to satisfy Assumption \ref{asp:alphaholder}.
		Let us now assume that the raw data have been anonymized by means of an arbitrary privacy mechanism $\mathrm{R}$ that is sequentially-interactive $\varepsilon$-semi-feature LDP. 
		$\mathrm{R}$ maps each of $(X_i,Y_i)$ to $Z_i$ based on $Z_1,\cdots, Z_{i-1}$. 
		Let $\mathrm{P}^{\theta n} $ and $\mathrm{R}\mathrm{P}^{\theta n} $ be the joint distribution of $\left((X_1, Y_1), \cdots, (X_n, Y_n)\right)$ and $(Z_1, \cdots,Z_n)$, respectively. 
		Let $\widehat{f}$ be the estimator drawn by using $(Z_1, \cdots,Z_n)$.  
		We are interested in the following quantity
		\begin{align*}
			\mathbb{E}_{\mathrm{R}\mathrm{P}^{\theta n} } \mathcal{R}_{L,\mathrm{P^{\theta}}} (\widehat{f})
			- \mathcal{R}_{L,\mathrm{P^{\theta}}}^*
			=
			\mathbb{E}_{\mathrm{R}\mathrm{P}^{\theta n} }\left[\int_{[0,1]^d}\left(\widehat{f}(x') - K^{\theta}(x')\right)^2 d\mathrm{P^{\theta}_X(x')}\right]. 
		\end{align*}
		For each fixed $\widehat{f}$, let 
		\begin{align*}
			\widehat{\theta} = {\arg \min}_{\theta \in \{0,1\}^{k^d}} \left( \mathcal{R}_{L,\mathrm{P^{\theta}}} (\widehat{f})
			- \mathcal{R}_{L,\mathrm{P^{\theta}}}^*\right)
		\end{align*}
		Thus, by a triangular decomposition, the risk has 
		\begin{align*}
			\mathcal{R}_{L,\mathrm{P^{\theta}}} (K^{\widehat{\theta}})
			- \mathcal{R}_{L,\mathrm{P^{\theta}}}^*  \leq &  \mathcal{R}_{L,\mathrm{P^{\theta}}} (\widehat{f})
			-   \mathcal{R}_{L,\mathrm{P^{\theta}}}^*+  \mathcal{R}_{L,\mathrm{P^{\widehat{\theta}}}} (\widehat{f}) - \mathcal{R}_{L,\mathrm{P^{\widehat{\theta}}}}^* \\
			\leq & 2 \left(  \mathcal{R}_{L,\mathrm{P^{\theta}}} (\widehat{f})
			-   \mathcal{R}_{L,\mathrm{P^{\theta}}}^* \right). 
		\end{align*}
		By a reordering of the decomposition of the domain $[0,1]^d = \cup_{\sigma} \times_{i=1}^d\left[\sigma^i / k,\left(\sigma^i+1\right) / k\right] := \cup_{\theta}B^{\theta}$, we can reduce the expected risk as
		\begin{align}\nonumber
			\mathbb{E}_{\mathrm{R}\mathrm{P}^{\theta n} } \mathcal{R}_{L,\mathrm{P^{\theta}}} (\widehat{f})
			- \mathcal{R}_{L,\mathrm{P^{\theta}}}^* \geq &  \frac{1}{2}\left(\mathbb{E}_{\mathrm{R}\mathrm{P}^{\theta n} }\mathcal{R}_{L,\mathrm{P^{\theta}}} (K^{\widehat{\theta}})
			- \mathcal{R}_{L,\mathrm{P^{\theta}}}^*\right)\\\nonumber
			= & \frac{1}{2}\mathbb{E}_{\mathrm{R}\mathrm{P}^{\theta n} } \left[\int_{[0,1]^d } \left(K^{\widehat{\theta}}(x') - K^{{\theta}}(x') \right)^2d x'\right] \\\label{equ:lowerbound1}
			= & \frac{1}{k^{2\alpha + d}} \mathbb{E}_{\mathrm{R}\mathrm{P}^{\theta n} }\left[\eins_{\widehat{\theta} = \theta}\right]\cdot \int_{[0,1]^d} K(x') dx' =  \frac{1}{k^{2\alpha + d}} \mathbb{E}_{\mathrm{R}\mathrm{P}^{\theta n} }\left[\eins_{\widehat{\theta} = \theta}\right]. 
		\end{align}

		For any pair of distribution $\mathrm{P}_1$ and $\mathrm{P}_2$, 
		let $KL$ be the Kullback-Leibler divergence of $\mathrm{P}_1$ and $\mathrm{P}_2$, i.e. 
		$KL(\mathrm{P}_1 \| \mathrm{P_2}) = \int \log \left(d\mathrm{P}_1(x) / d\mathrm{P}_2\right) d\mathrm{P}_1(x)$. 
		Also, let the total variation distance of $\mathrm{P}_1$ and $\mathrm{P}_2$ be $V(\mathrm{P}_1 , \mathrm{P_2}) = \frac{1}{2} \int |d\mathrm{P}_1(x)  - d\mathrm{P}_2(x) |$.
		To utilize the machinery in \citet{duchi2018minimax}, we need to bound the quantity $KL(\mathrm{R}\mathrm{P}^{\theta_1n}\| \mathrm{R}\mathrm{P}^{\theta_2n} )$ for arbitrary $\theta_1$ and $\theta_2$. 
		Tensorization yields
		\begin{align}\label{equ:tensorizationofKL}
			KL(\mathrm{R}\mathrm{P}^{\theta_1n}\| \mathrm{R}\mathrm{P}^{\theta_2n} ) = \sum_{i=1}^n KL(\mathrm{R}_i\mathrm{P}^{\theta_1}\| \mathrm{R}_i\mathrm{P}^{\theta_2} ).
		\end{align}
		Let $p^{\theta\text{pub}}$ be the probability density of $(X^{s^*+1}, X^{s^*+2}, \cdots, X^{d})$. 
		Let $\mathrm{P}^{\theta\text{priv}}$ be the conditional probability measure of $(X^{1}, X^{2}, \cdots, X^{s^*}, Y)$ conditional on $(X^{s^*+1}, X^{s^*+2}, \cdots, X^{d})$ and $p^{\theta\text{priv}}$ be its density function. 
		And let $r_i$ be the density of mechanism $\mathrm{R}_i$ condition on $(X_i, Y_i)$. 
		Since the estimated $\widehat{\theta}$ is drawn on both $Z_i$ and $X^{\text{pub}}_i$, w.o.l.g. we write the output of mechanism $\mathrm{R}$ as $(Z_i, X^{\text{pub}}_i)$. 
		Then the output density function is $m_i^{\theta}(z_i, x^{\text{pub}}_i) = \int_{\mathcal{X}^{\text{priv}} \times \mathcal{Y}} r_i(z_i  | x_i^{\text{priv}}, x_i^{\text{pub}}, y_i)p^{\theta\text{priv}}(x^{\text{priv}}_i, y_i | x_i^{\text{pub}})p^{\theta\text{pub}}(x^{\text{pub}}_i) d x_i^{\text{priv}} dy_i $.
		Then we have 
		\begin{align}\nonumber
			KL(\mathrm{R}_i\mathrm{P}^{\theta_1}\| \mathrm{R}_i\mathrm{P}^{\theta_2} ) = & \int_{\mathcal{X}^{\text{pub}}} \int_{\mathcal{S} }\log \left(\frac{m_i^{\theta_1}(z_i, x^{\text{pub}}_i)}{m_i^{\theta_2}(z_i, x^{\text{pub}}_i)}\right)  m_i^{\theta_1}(z_i, x^{\text{pub}}_i) d z_i d x^{\text{pub}}_i\\
			= &  \int_{\mathcal{X}^{\text{pub}}}\underbrace{\left[\int_{\mathcal{S} }\log \left(\frac{m_i^{\theta_1}(z_i| x^{\text{pub}}_i) p^{\theta_1\text{pub}}(x_i^{\text{pub}})}{m_i^{\theta_2}(z_i| x^{\text{pub}}_i) p^{\theta_2\text{pub}}(x_i^{\text{pub}})}\right)  m_i^{\theta_1}(z_i| x^{\text{pub}}_i) d z_i\right]}_{(**)} p^{\theta_1\text{pub}}(x_i^{\text{pub}}) d x^{\text{pub}}_i. \label{equ:lowerbound2}
		\end{align}
		Our construction yields that 
		\begin{align*}
			(**) = \int_{\mathcal{S} }\log \left(\frac{m_i^{\theta_1}(z_i| x^{\text{pub}}_i) }{m_i^{\theta_2}(z_i| x^{\text{pub}}_i) }\right)  m_i^{\theta_1}(z_i| x^{\text{pub}}_i) d z_i = KL\left((\mathrm{R}_i\mathrm{P}^{\theta_1})_{Z_i | X^{\text{pub}}_i}\| (\mathrm{R}_i\mathrm{P}^{\theta_2})_{Z_i | X^{\text{pub}}_i} \right)
		\end{align*}
		By Definition \ref{def:ldp},  $Z_i$ given $X_i^{\text{pub}}$ is actually locally differentially private (as rigorously defined in \citet{duchi2018minimax, gyorfi2022rate}) w.r.t. $(X^{\text{priv}}_i, Y_i)$. 
		Thus, applying Theorem 1 in \citet{duchi2018minimax}, we have 
		\begin{align*}
			(**) \leq  4\left(e^{\varepsilon} - 1\right)^2V^2\left(\mathrm{P}^{\theta_1\text{priv}}, \mathrm{P}^{\theta_2\text{priv}}\right). 
		\end{align*}
		Bringing in our construction where the density on $[0,1]^d \times [K^{\theta}(x) - 1/2, K^{\theta}(x) + 1/2]$ is 1 and otherwise zero, we get
		\begin{align*}
			(**)  \leq & 4\left(e^{\varepsilon} - 1\right)^2\left( \int_{\mathcal{X}^{\text{priv}}} |K^{\theta_1}(x^{\text{priv}}_i, x^{\text{pub}}_i) - K^{\theta_2}(x^{\text{priv}}_i, x^{\text{pub}}_i)| dx^{\text{priv}}_i\right)^2\\
			= & 4\left(e^{\varepsilon} - 1\right)^2 \cdot k^{-2\alpha -2 s^* }  \cdot \left( 2 \cdot \int_{[0,1]^{s^*}} |K(x, k x^{\text{pub}}_i - \sigma_2) | dx \right)^2. 
		\end{align*}
		Bringing back this into \eqref{equ:lowerbound2} and apply Cauchy's inequality, there holds
		\begin{align*}
			KL(\mathrm{R}_i\mathrm{P}^{\theta_1}\| \mathrm{R}_i\mathrm{P}^{\theta_2} )
			\leq &16\left(e^{\varepsilon} - 1\right)^2 \cdot k^{-2\alpha -2 s^* } \int_{\mathcal{X}^{\text{pub}}}   \int_{[0,1]^{s^*}} |K(x, k x^{\text{pub}}_i - \sigma_2) |^2 dx   p^{\theta_1\text{pub}}(x_i^{\text{pub}}) d x^{\text{pub}}_i
			\\ 
			\leq & 16\left(e^{\varepsilon} - 1\right)^2 \cdot k^{-2\alpha -2 s^*  - (d-s^*)} \int_{[0,1]^d} |K(x)|dx \leq 32 \left(e^{\varepsilon} - 1\right)^2 \cdot k^{-2\alpha - s^*  - d}.
		\end{align*}
		Bringing back this result into \eqref{equ:tensorizationofKL} leads to 
		\begin{align*}
			KL(\mathrm{R}\mathrm{P}^{\theta_1n}\| \mathrm{R}\mathrm{P}^{\theta_2n} )\leq 32 \left(e^{\varepsilon} - 1\right)^2 \cdot k^{-2\alpha - s^*  - d}n.
		\end{align*}
		Consequently, taking a $k = \left(32 n(e^{\varepsilon} - 1)^2\right)^{\frac{1}{2\alpha + s^* + d}}$ yields a $	KL(\mathrm{R}\mathrm{P}^{\theta_1n}\| \mathrm{R}\mathrm{P}^{\theta_2n} )\leq 1$. 
		Applying Theorem 2.12, Statement (iv) of \citet{tsybakov2008introduction} gives us $\mathbb{E}_{\mathrm{R}\mathrm{P}^{\theta n} }\left[\eins_{\widehat{\theta} = \theta}\right] \geq k^d / 8$, which is combined with \eqref{equ:lowerbound1} and leads to 
		\begin{align*}
			\mathbb{E}_{\mathrm{R}\mathrm{P}^{\theta n} } \mathcal{R}_{L,\mathrm{P^{\theta}}} (\widehat{f})
			- \mathcal{R}_{L,\mathrm{P^{\theta}}}^* \gtrsim \left(n(e^{\varepsilon} - 1)^2\right)^{\frac{2\alpha}{2\alpha + s^* + d}}. 
		\end{align*}
	\end{proof}

	\begin{align*}
		\sum_{i=1}^n \mathcal{E}^{\frac{2\alpha + d + \sum_{j=1}^d W_i^j}{2\alpha}} \leq \frac{1}{(e^{\varepsilon} - 1)^2 } 
	\end{align*}
	
	\begin{align*}
		\sum_{i=1}^n \mathcal{E}^{ - \frac{2\alpha + d + \sum_{j=1}^d W_i^j}{2\alpha}} = n^2\varepsilon^2
	\end{align*}

	\subsection{Proof of Theorem \ref{thm:utility}} \label{app:proofofutility}

	We first define the intermediate quantity 
	\begin{align}\label{equ:populationestimator}
		\overline{f}(x) = \sum_{A_j, B_k}   \eins_{A_j\times B_k}(x) \frac{\int_{A_j\times B_k} f^*(x) d\mathrm{P}_X(x)  } {\int_{A_j\times B_k}  d\mathrm{P}_X(x) }.
	\end{align}
	We rely on the following decomposition.
	\begin{align}\nonumber
		\mathcal{R}_{L,\mathrm{P}}(\tilde{f}) - \mathcal{R}_{L,\mathrm{P}}(f^*) = & \underbrace{ \mathcal{R}_{L,\mathrm{P}}(\tilde{f}) - \mathcal{R}_{L,\mathrm{P}}(f)}_{\textbf{Privatized Error}} + \underbrace{ \mathcal{R}_{L,\mathrm{P}}(f) - \mathcal{R}_{L,\mathrm{P}}(\overline{f})}_{\textbf{Sample Error}}\\ & + \underbrace{ \mathcal{R}_{L,\mathrm{P}}(\overline{f}) - \mathcal{R}_{L,\mathrm{P}}(f^*)}_{\textbf{Approximation Error}}. \label{equ:decompositionofexcessrisk}
	\end{align}
	where the \texttt{HistOfTree} estimator $\tilde{f}$, the non-private partition-based estimator $f$, and the population partition-based estimator $\overline{f}$ are defined in \eqref{equ:ourestimatorpersonal}, \eqref{equ:nonprivatedecisiontreeestimator}, and \eqref{equ:populationestimator}, respectively. 
	Loosely speaking, the first error term quantifies the depravation brought by adding privacy noises to the estimator, which we call the privatized error. The second term corresponds to the expected estimation error brought by the randomness of the data, which we call the sample error. The last term is called approximation error, which arises due to the limited approximation capacity of piecewise constant functions.
	The following three lemmas provide bounds for each of the three errors.

	\begin{lemma}[\textbf{Bounding privatised error}]\label{lem:boundingofprivatisederror}
		Let the \texttt{HistOfTree} estimator $\tilde{f}$ and the non-private partition-based estimator $f$ be defined in \eqref{equ:ourestimatorpersonal} and \eqref{equ:nonprivatedecisiontreeestimator}, respectively. 
		For $\pi = \pi^{\text{priv}} \otimes \pi^{\text{pub}} $, let $\pi^{\text{priv}}$ be a histogram partition with $t$ bins and $\pi^{\text{pub}}$ be generated by Algorithm \ref{alg:partition} with depth $p$.  
		Let Assumption \ref{asp:alphaholder} hold. 
		Suppose $t^s \cdot 2^p \cdot \log n \lesssim {n }$.
		Then, there holds
		\begin{align*}
			\mathcal{R}_{L,\mathrm{P}}(\tilde{f}) - \mathcal{R}_{L,\mathrm{P}}(f)  \lesssim \frac{ t^{2s} \cdot 2^{p} \cdot \log n \cdot  \frac{1 }{n}\sum_{i=1}^n 2^{ \sum_{\ell = s + 1}^d  W_i^{\ell} \cdot p / (d - s) } }{ n\varepsilon^2}
		\end{align*}
		with probability $\mathrm{P}^n\otimes \mathrm{R}^n$ at least $1-4/n^2$.
	\end{lemma}

	\begin{lemma}[\textbf{Bounding sample error}]\label{lem:boundingofsampleerror}
		Let the non-private partition-based estimator $f$ and the population partition-based estimator $\overline{f}$ be defined in \eqref{equ:nonprivatedecisiontreeestimator} and \eqref{equ:populationestimator}, respectively. 
		For $\pi = \pi^{\text{priv}} \otimes \pi^{\text{pub}} $, let $\pi^{\text{priv}}$ be a histogram partition with $t$ bins and $\pi^{\text{pub}}$ be generated by Algorithm \ref{alg:partition} with depth $p$.  
		Let Assumption \ref{asp:alphaholder} hold. 
		Suppose $t^s \cdot 2^p \cdot \log n \lesssim {n }$.
		Then, there holds
		\begin{align*}
			\mathcal{R}_{L,\mathrm{P}}(f) - \mathcal{R}_{L,\mathrm{P}}(\overline{f})  \lesssim \frac{ t^{s} \cdot 2^{p} \cdot \log n }{ n}
		\end{align*}
		with probability $\mathrm{P}^n\otimes \mathrm{R}^n$ at least $1-4/n^2$.
	\end{lemma}
	
	\begin{lemma}[\textbf{Bounding approximation error}]\label{lem:boundingofapproximationerror}
		Let the population partition-based estimator $\overline{f}$ be defined in \eqref{equ:populationestimator}. 
		For $\pi = \pi^{\text{priv}} \otimes \pi^{\text{pub}} $, let $\pi^{\text{priv}}$ be a histogram partition with $t$ bins and $\pi^{\text{pub}}$ be generated by Algorithm \ref{alg:partition} with depth $p$.  
		Let Assumption \ref{asp:alphaholder} hold. 
		Then there holds
		\begin{align*}
			\mathcal{R}_{L,\mathrm{P}}(\overline{f}) - \mathcal{R}_{L,\mathrm{P}}({f}^*) \lesssim \;  t^{-2\alpha}  +  2^{- 2\alpha p / (d - s)  } .
		\end{align*}
	\end{lemma}

	\begin{proof}[\textbf{Proof of Theorem  \ref{thm:utility}}]
		The privacy of $\tilde{f}$ follows from Proposition \ref{prop:privacy}. 
		We then focus on the excess risk upper bound. 
		Bringing Proposition \ref{lem:boundingofprivatisederror},  \ref{lem:boundingofsampleerror}, and \ref{lem:boundingofapproximationerror} into the decomposition 
		\eqref{equ:decompositionofexcessrisk}, we have 
		\begin{align}\nonumber
			&	\mathcal{R}_{L,\mathrm{P}}(\tilde{f}) - \mathcal{R}_{L,\mathrm{P}}^*\\
			\lesssim  & \frac{ t^{2s} \cdot 2^{p} \cdot \log n \cdot  \frac{1 }{n}\sum_{i=1}^n 2^{ \sum_{\ell = s + 1}^d  W_i^{\ell} \cdot p / (d - s) } }{ n\varepsilon^2} + \frac{ t^{s} \cdot 2^{p} \cdot \log n }{ n} +  t^{-2\alpha}  +  2^{- 2\alpha p / (d - s)  }  \label{equ:decompositionupperbound1}
		\end{align}
		holds with probability $\mathrm{P}^n\otimes \mathrm{R}^n$ at least $1-8/n^2$. 
		For any $\varepsilon^2 \lesssim t^s  \cdot  \frac{1 }{n}\sum_{i=1}^n 2^{ \sum_{\ell = s + 1}^d  W_i^{\ell} \cdot p / (d - s) } $, the first term dominants the second term.
		We will specify this range of $\varepsilon$ later. 
		Since we take $t \asymp 2^{p/ (d-s)}$, we have $ t^{-2\alpha} \asymp2^{- 2\alpha p / (d - s)  }  $ and $t^{2s} \cdot 2^p \asymp 2^{p (d+ s) / (d-s)}$. 
		Therefore, \eqref{equ:decompositionupperbound1} becomes 
		\begin{align*}
			\mathcal{R}_{L,\mathrm{P}}(\tilde{f}) - \mathcal{R}_{L,\mathrm{P}}^*
			\lesssim   \frac{ 2^{p (d+ s) / (d-s)} \cdot \log n \cdot  \frac{1 }{n}\sum_{i=1}^n 2^{ \sum_{\ell = s + 1}^d  W_i^{\ell} \cdot p / (d - s) } }{ n\varepsilon^2} +   2^{- 2\alpha p / (d - s)  }, 
		\end{align*}
		which recover the equation \eqref{equ:selectionofp}.
		Recalling the way we select $p^*$, the upper bound becomes 
		\begin{align*}
			\mathcal{R}_{L,\mathrm{P}}(\tilde{f}) - \mathcal{R}_{L,\mathrm{P}}^*
			\lesssim  	 \frac{ 2^{p^* (d+ s) / (d-s)} \cdot \log n \cdot 2^{p^* \lambda^*} }{ n\varepsilon^2} +   2^{- 2\alpha p^* / (d - s)  }.
		\end{align*}
		The minimizer is taken at the $p^*$ that match the two terms, which is $2^{p^*} \asymp \left(n\varepsilon^2 / \log n\right)^{(d-s) / (2\alpha + d + s + \lambda^* (d - s))}$. 
		In this case, the upper bound is 
		\begin{align*}
			\mathcal{R}_{L,\mathrm{P}}(\tilde{f}) - \mathcal{R}_{L,\mathrm{P}}^*
			\lesssim  	\left(\frac{\log n}{n\varepsilon^2}\right)^{\frac{2\alpha }{2\alpha + d + s + \lambda^* (d - s)}}. 
		\end{align*}
		The definition of $\lambda^*$ yields that it is less than 1. 
		Thus, 
		\begin{align*}
			t^s \asymp 2^{\frac{s p^{*}}{d - s}}\asymp \left(\frac{\log n}{n\varepsilon^2}\right)^{-\frac{s}{2\alpha + d + s + \lambda^* (d - s)}} \gtrsim \left(\frac{\log n}{n\varepsilon^2}\right)^{-\frac{s}{2\alpha + 2d }}.
		\end{align*}
		As a result, it suffices to require $\left(\frac{\log n}{n\varepsilon^2}\right)^{-\frac{s}{2\alpha + 2d }}\gtrsim \varepsilon^2$, which yields $\varepsilon \lesssim \left(n / \log n\right)^{\frac{s}{\alpha + d - s}}$. 
	\end{proof}

	\subsection{Proofs of Results in Section \ref{app:proofofutility}}\label{sec:proofoferroranalysis}
	
	To prove lemmas in Section \ref{app:proofofutility}, we first present several technical results. 
	
	\begin{lemma}\label{lem:hoeffding}
		Suppose $\zeta_{i}, i= 1,\cdots, n$ are independent random variables such that $a_i \leq \zeta_i\leq b_i$. Then there holds
		\begin{align*}
			\mathbb{P}\left[\left|\frac{1}{n}\sum_{i=1}^n \zeta_{i}  - \mathbb{E}\frac{1}{n}\sum_{i=1}^n \zeta_{i}\right| \geq t \right] \leq 2 e^{ - \frac{2 n^2 t^2 }{\sum_{i=1}^n(b_i-a_i)^2}} 
		\end{align*}
		for any $t>0$. 
	\end{lemma}

	\begin{proof}[\textbf{Proof of Lemma \ref{lem:hoeffding}}]
		The conclusion follows from Example 2.4 and Proposition 2.5 in \citep{wainwright2019high}.   
	\end{proof}

	\begin{lemma}\label{lem:laplacebound}
		Suppose $\xi_{i}, i = 1,\cdots, n$ are independent sub-exponential random variables with parameters $(\nu_i, \beta_i)$ \cite[Definition 2.9]{wainwright2019high}. Then there holds
		\begin{align*}
			\mathbb{P}\left[\left|\frac{1}{n}\sum_{i=1}^n \xi_{i} - \mathbb{E} \frac{1}{n}\sum_{i=1}^n \xi_{i} \right| \geq t  \right] \leq 2 e^{ - \frac{n t^2 }{2 \nu_*^2 } }. 
		\end{align*}
		for any $0< t < \nu_*^2 / (n\beta_*)$, where $\nu_* = \sqrt{\sum_{i=1}^n\nu_i^2}$ and $\beta_* = \max_{i=1,\cdots, n}\beta_i$.  Moreover, a standard Laplace random variable is sub-exponential with parameters $(\sqrt{2},1)$, and $a \xi_i$ has $(a\nu_i, a\beta_i)$ for any positive $a$.
	\end{lemma}

	\begin{proof}[\textbf{Proof of Lemma \ref{lem:laplacebound}}]
		The conclusion follows from (2.18) in \citep{wainwright2019high}.  
	\end{proof}

	\begin{lemma}\label{VCIndex}
		Let $\tilde{\mathcal{A}}$ be the collection of all cells $\times_{i=1}^d [a_i,b_i]$ in $\mathbb{R}^d$. The VC index of $\tilde{\mathcal{A}}$ equals $2d+1$. Moreover, for all $0<\varepsilon<1$, there exists a universal constant $C$ such that
		\begin{align*}
			\mathcal{N}(\eins_{\tilde{\mathcal{A}}}, \|\cdot\|_{L_1(Q)}, \varepsilon)\leq C(2d+1)(4e)^{2d+1}(1/\varepsilon)^{2d}.
		\end{align*}
	\end{lemma}
	
	\begin{proof}[Proof of Lemma \ref{VCIndex}]
		The first result of the VC index follows from Example 2.6.1 in \citep{vandervaart1996weak}. The second result of covering number follows directly from Theorem 9.2 in \citep{Kosorok2008introduction}.
	\end{proof}

	\begin{lemma}\label{lem:cellbound}
		For $\pi = \pi^{\text{priv}} \otimes \pi^{\text{pub}} $, let $\pi^{\text{priv}}$ be a histogram partition with $t$ bins and $\pi^{\text{pub}}$ be generated by Algorithm \ref{alg:partition} with depth $p$.  
		Suppose Assumption \ref{asp:boundedmarginal} holds. 
		Then for any $\{(X_1,Y_1),\cdots, (X_n,Y_n)\}$ drawn i.i.d. from $\mathrm{P}$ and any $A \in \pi$, there holds
		\begin{align*}
			\left|\frac{1}{n}\sum_{i=1}^{n}  \mathbf{1}\{X_i \in A  \} - \int_{A}d\mathrm{P}_X(x') \right| \leq \sqrt{\frac{\overline{c}\cdot t^{-s} \cdot  2^{1-p}(4d+5)\log n}{n}} +\frac{2(4d+5)\log n}{3n}+\frac{4}{n}
		\end{align*}
		with probability $\mathrm{P}^n$ at least $1-1/n^2$.
		Here we use $A$ instead of $A\times B$ for notation simplicity. 
	\end{lemma}

	\begin{proof}[\textbf{Proof of Lemma \ref{lem:cellbound}}]
		In the subsequent proof, we define the empirical measure $\mathrm{D} := \frac{1}{n} \sum_{i=1}^n \delta_{(X_i,Y_i)}$ given samples $ \{(X_1, Y_1), \cdots, (X_n,Y_n)\}$, where $\delta$ is the Dirac function. 
		Let 
		$\tilde{\mathcal{A}}$ be the collection of all cells $\times_{i=1}^d [a_i,b_i]$ in $\mathbb{R}^d$.
		Applying Lemma \ref{VCIndex} with $Q:=(\mathrm{D}_X+\mathrm{P}_X)/2$, there exists an $\varepsilon$-net $\{\tilde{A}_k\}_{k=1}^K\subset \tilde{\mathcal{A}}$ with 
		\begin{align}\label{equ::Kcover}
			K\leq C(2d+1)(4e)^{2d+1}(1/\varepsilon)^{2d}
		\end{align}
		such that
		for any $A\in \pi$, there exist some $k\in \{1,\ldots,K\}$ such that
		\begin{align*}
			\|\eins\{x\in A\}-\eins\{x\in \tilde{A}_k\}\|_{L_1((\mathrm{D}_X+\mathrm{P}_X)/2)}\leq \varepsilon,
		\end{align*}
		Since 
		\begin{align*}
			&\|\eins\{x\in A\}-\eins\{x\in\tilde{A}_k\}\|_{L_1((\mathrm{D}_X+\mathrm{P}_X)/2)}\\
			= & 1/2\cdot \|\eins\{x\in A\}-\eins\{x\in \tilde{A}_k\}\|_{L_1(\mathrm{D}_X)}+1/2\cdot \|\eins\{x\in A\}-\eins\{x\in \tilde{A}_k\}\|_{L_1(\mathrm{P}_X)},
		\end{align*}
		we get
		\begin{align}\label{equ::einsapj}
			\|\eins\{x\in A\}-\eins\{x\in \tilde{A}_k\}\|_{L_1(\mathrm{D}_X)}\leq 2\varepsilon, \quad  \|\eins\{x\in A\}-\eins\{x\in \tilde{A}_k\}\|_{L_1(\mathrm{P}_X)}\leq 2\varepsilon.
		\end{align}
		Consequently, by the definition of the covering number and the triangle inequality, for any $A\in\pi $, there holds
		\begin{align*}
			&\biggl|\frac{1}{n}\sum_{i=1}^n\eins\{x\in A\}(X_i)- \int_{\tilde{A}_p^j}d\mathrm{P}_X(x') \biggr|\\
			\leq &
			\biggl|\frac{1}{n} \sum_{i=1}^n \eins\{x\in \tilde{A}_k\}(X_i) - \int_{\tilde{A}_k}d\mathrm{P}_X(x')\biggr|+\|\eins\{x\in A\}-\eins\{x\in \tilde{A}_k\}\|_{L_1(\mathrm{D}_X)}\\ &  + \|\eins\{x\in A\}-\eins\{x\in \tilde{A}_k\}\|_{L_1(\mathrm{P}_X)}
			\leq \biggl|\frac{1}{n} \sum_{i=1}^n \eins\{x\in \tilde{A}_k\}(X_i)-\int_{\tilde{A}_k}d\mathrm{P}_X(x')\biggr| + 4\varepsilon.
		\end{align*}
		Therefore, we get
		\begin{align}\label{equ::Isupjip}
			\sup_{j\in \mathcal{I}} \,\biggl|\frac{1}{n}\sum_{i=1}^n \eins\{x\in A\}(X_i)- \int_{\tilde{A}_p^j}d\mathrm{P}_X(x') \biggr|
			\leq \sup_{1\leq k\leq K} \,\biggl|\frac{1}{n} \sum_{i=1}^n \eins\{x\in\tilde{A}_k\}(X_i)-\int_{\tilde{A}_k}d\mathrm{P}_X(x')\biggr| + 4\varepsilon.
		\end{align}
		For any fixed $1\leq k\leq K$, let the random variable $\xi_i$ be defined by $\xi_i:=\eins\{X_i\in\tilde{A}_k\}-\int_{\tilde{A}_k}d\mathrm{P}_X(x')$.
		Then we have $\mathbb{E}_{\mathrm{P}_X}\xi_i=0$, $\|\xi\|_\infty\leq 1$, and $\mathbb{E}_{\mathrm{P}_X}\xi_i^2\leq \int_{\tilde{A}_k}d\mathrm{P}_X(x')$. 
		According to Assumption \ref{asp:boundedmarginal}, there holds $\mathbb{E}_{\mathrm{P}_X}\xi_i^2\leq \overline{c}\cdot t^{-s} \cdot 2^{-p} $. 
		Applying Bernstein's inequality, we obtain
		\begin{align*}
			\biggl|\frac{1}{n} \sum_{i=1}^n \eins\{X_i\in \tilde{A}_k\}-\int_{\tilde{A}_k}d\mathrm{P}_X(x')\biggr|\leq \sqrt{\frac{\overline{c}\cdot t^{-s} \cdot 2^{1-p} \cdot \tau}{n}} +\frac{2\tau\log n}{3n}
		\end{align*}
		with probability $\mathrm{P}^n$ at least $1-2e^{-\tau}$. 
		Then the union bound together with the covering number estimate \eqref{equ::Kcover} implies that
		\begin{align*}
			\sup_{1\leq k\leq K}\, \biggl|\frac{1}{n} \sum_{i=1}^n \eins\{X_i\in\tilde{A}_k\}-\int_{\tilde{A}_k}d\mathrm{P}_X(x')\biggr|\leq\sqrt{\frac{\overline{c}\cdot t^{-s} \cdot  2^{1-p}(\tau+\log (2K))}{n}} +\frac{2(\tau+\log (2K))\log n}{3n}
		\end{align*}
		with probability $\mathrm{P}^n$ at least $1-e^{-\tau}$. Let $\tau=2\log n$ and $\varepsilon=1/n$. Then for any $n> N_1:=(2C)\wedge (2d+1)\wedge (4e)$, we have $\tau+\log (2K)=2\log n+ \log (2C)+\log (2d+1)+(2d+1)\log (4e)+ 2d\log n\leq (4d+5)\log n$. Therefore, we have
		\begin{align}\label{equ::i1upperk}
			\sup_{1\leq k\leq K}\, \biggl|\frac{1}{n} \sum_{i=1}^n \eins\{X_i \in \tilde{A}_k\}-\int_{\tilde{A}_k}d\mathrm{P}_X(x')\biggr|\leq\sqrt{\frac{\overline{c}\cdot t^{-s} \cdot  2^{1-p}(4d+5)\log n}{n}} +\frac{2(4d+5)\log n}{3n}
		\end{align}
		with probability $\mathrm{P}^n$ at least $1-1/n^2$. This together with \eqref{equ::Isupjip} yields that
		\begin{align*}
			\sup_{j\in \mathcal{I}} \,\biggl|\frac{1}{n}\sum_{i=1}^n \eins\{x\in A\}- \int_{\tilde{A}_p^j}d\mathrm{P}_X(x') \biggr|\leq \sqrt{\frac{\overline{c}\cdot t^{-s} \cdot  2^{1-p}(4d+5)\log n}{n}} +\frac{2(4d+5)\log n}{3n}+\frac{4}{n}.
		\end{align*}
	\end{proof}
	
	\begin{lemma}\label{lem:cellboundY}
		For $\pi = \pi^{\text{priv}} \otimes \pi^{\text{pub}} $, let $\pi^{\text{priv}}$ be a histogram partition with $t$ bins and $\pi^{\text{pub}}$ be generated by Algorithm \ref{alg:partition} with depth $p$.  
		Suppose Assumption \ref{asp:boundedmarginal} holds. 
		Then for any $\{(X_1,Y_1),\cdots, (X_n,Y_n)\}$ drawn i.i.d. from $\mathrm{P}$ and any $A \in \pi$, there holds
		\begin{align*}
			& \left|\frac{1}{n}\sum_{i=1}^{n}  Y_i \mathbf{1}\{X_i \in A(x)\} - \int_{A(x)}f^*(x')d\mathrm{P}_X(x') \right|\\ \leq &  M\sqrt{\frac{\overline{c}\cdot t^{-s}\cdot  2^{1-p}(4d+5)\log n}{n}} +\frac{2M(4d+5)\log n}{3n}+\frac{4M}{n}
		\end{align*}
		with probability $\mathrm{P}^n$ at least $1-1/n^2$.
		Here we use $A$ instead of $A\times B$ for notation simplicity. 
	\end{lemma}

	\begin{proof}[\textbf{Proof of Lemma \ref{lem:cellboundY}}]
		Let 
		$\tilde{\mathcal{A}}$ be the collection of all cells $\times_{i=1}^d [a_i,b_i]$ in $\mathbb{R}^d$.
		Then there exists an $\varepsilon$-net $\{\tilde{A}_k\}_{k=1}^K\subset \tilde{\mathcal{A}}$ with $K$ bounded by \eqref{equ::Kcover} such that 
		for any $j\in \mathcal{I}$,
		\eqref{equ::einsapj} holds
		for some $k\in \{1,\ldots,K\}$. 
		Consequently, by the definition of the covering number and the triangle inequality, for any $A\in \pi$, there holds
		\begin{align}
			&\biggl|\sum_{i=1}^n\eins\{X_i\in A\}Y_i -\int_{A}f^*(x') d\mathrm{P}_X(x')\biggr|\nonumber\\
			&\leq \biggl|\sum_{i=1}^n \eins\{X_i\in\tilde{A}_k\}Y_i -\int_{\tilde{A}_k} f^*(x') d\mathrm{P}_X(x')\biggr| \nonumber\\
			& +
			\int_{\mathbb{R}^d} \bigl|\eins\{x'\in A\}-\eins\{x'\in\tilde{A}_k\}\bigr| \bigl|f^*(x')\bigr| d\mathrm{P}_X(x') +\sum_{i=1}^n \bigl|\eins\{X_i\in\tilde{A}_k\} - \eins\{X_i\in A\}\bigr| \bigl|Y_i\bigr|\nonumber\\
			& \leq \biggl|\sum_{i=1}^n \eins\{X_i\in \tilde{A}_k\}Y_i -\int_{\tilde{A}_k} f^*(x') d\mathrm{P}_X(x')\biggr|\nonumber\\
			&+\max_{1\leq i\leq n}|Y_i| \cdot \|\eins\{x\in A\}-\eins\{x\in \tilde{A}_k\}\|_{L_1(\mathrm{D}_X)} + M\cdot \|\eins\{x\in A\}-\eins\{x\in\tilde{A}_k\}\|_{L_1(\mathrm{P}_X)}\nonumber\\
			&\leq \biggl|\sum_{i=1}^n \eins_{\tilde{A}_k}(X_i)Y_i -\int_{\tilde{A}_k} f^*(x') d\mathrm{P}_X(x')\biggr|+4M\varepsilon.\label{equ::Isupjip1}
		\end{align}
		where the last inequality follow from the condition $\mathcal{Y}\subset [-M,M]$.
		
		For any fixed $1\leq k\leq K$, let the random variable $\tilde{\xi}_i$ be defined by $\tilde{\xi}_i:=\eins\{X_i\in \tilde{A}_k\}Y_i-\int_{\tilde{A}_k} f^*(x')\, d\mathrm{P}_X(x')$.
		Then we have $\mathbb{E}_{\mathrm{P}}\tilde{\xi}_i=0$, $\|\xi\|_\infty\leq M + 1$, and $\mathbb{E}_{\mathrm{P}}\tilde{\xi}_i^2\leq M^2 \int_{\tilde{A}_k}d\mathrm{P}(x')$. According to Assumption \ref{asp:boundedmarginal}, there holds 
		$\mathbb{E}_{\mathrm{P}}\tilde{\xi}_i^2\leq M^2\cdot \overline{c}\cdot t^{-s}\cdot 2^{-p}$.
		Applying Bernstein's inequality, we obtain
		\begin{align*}
			\biggl|\sum_{i=1}^n \eins\{X_i\in\tilde{A}_k\}Y_i -\int_{\tilde{A}_k} f^*(x') d\mathrm{P}_X(x')\biggr| \leq \sqrt{\frac{M^2\cdot \overline{c}\cdot t^{-s} \cdot 2^{1-p}\cdot \tau}{n}} +\frac{2M\tau\log n}{3n}
		\end{align*}
		with probability $\mathrm{P}^n$ at least $1-2e^{-\tau}$. 
		Similar to the proof of Lemma \ref{lem:cellbound}, one can show that for any $n\geq N_1$, there holds
		\begin{align*}
			\sup_{1\leq k\leq K} \biggl|\sum_{i=1}^n \eins\{X_i \in \tilde{A}_k\}Y_i -\int_{\tilde{A}_k} f^*(x') d\mathrm{P}_X(x')\biggr| \leq M\sqrt{\frac{\overline{c}\cdot t^{-s} \cdot 2^{1-p}\cdot \tau}{n}} +\frac{2M\tau\log n}{3n}
		\end{align*}
		with probability $\mathrm{P}^n$ at least $1-1/n^2$. This together with \eqref{equ::Isupjip1} yields that
		\begin{align}\label{equ::i2bound}
			&\biggl|\sum_{i=1}^n\eins\{X_i\in A\}Y_i -\int_{A_p^j}f^*(x') d\mathrm{P}_X(x')\biggr|\\\leq& M\sqrt{\frac{\overline{c}\cdot t^{-s} \cdot 2^{1-p}(4d+5)\log n}{n}} +\frac{2M(4d+5)\log n}{3n}+\frac{4M}{n}.\nonumber
		\end{align}
	\end{proof}

	\begin{lemma}\label{lem::treeproperty}
		For $\pi = \pi^{\text{priv}} \otimes \pi^{\text{pub}} $, let $\pi^{\text{priv}}$ be a histogram partition with $t$ bins and $\pi^{\text{pub}}$ be generated by Algorithm \ref{alg:partition} with depth $p$.
		Then for any $A\times B :=\times_{i=1}^d [a_i,b_i] \in \pi$, there holds
		\begin{align*}
			2^{-2}\left(\sqrt{d - s }\cdot 2^{- p / (d - s)} + \sqrt{s} \; t^{-1} \right)\leq \mathrm{diam}(A\times B)\leq 2 \left(\sqrt{d - s }\cdot 2^{- p / (d - s)} + \sqrt{s} \; t^{-1} \right). 
		\end{align*}
	\end{lemma}
	
	\begin{proof}[\textbf{Proof of Lemma \ref{lem::treeproperty}}]
		For each $A\times B$, $s$ edges of $A$ have length $t^{-1}$. 
		The rest $d - s$ edges of $B$ are decided by the max-edge partition rule.
		When the depth of the tree $p$ is a multiple of dimension $d-s$, each cell of the tree partition is a high-dimensional cube with a side length $2^{-p/(d - s)}$.
		On the other hand, when the depth of the tree $p$ is not a multiple of dimension $d - s$, we consider the max-edge tree partition with depth $\lfloor p/(d - s)\rfloor$ and $\lceil p/(d - s)\rceil$, whose corresponding side length of the higher dimensional cube is $2^{-\lfloor p/(d- s)\rfloor}$ and $2^{-\lceil p/(d- s) \rceil}$. 
		Note that in the splitting procedure of max-edge partition, the side length of each sub-rectangle decreases monotonically with the increase of $p$, so the side length of a random tree partition cell is between $2^{-\lceil p/(d-s) \rceil}$ and $2^{-\lfloor p/(d - s) \rfloor}$. This implies that
		\begin{align*}
			\sqrt{d - s}\cdot 2^{-\lceil p/(d-s) \rceil}\leq \mathrm{diam}(B)\leq \sqrt{d-s}\cdot 2^{-\lfloor p/(d-s) \rfloor}
		\end{align*}
		Since $p/(d-s)-1\leq \lfloor p/(d-s)\rfloor\leq \lceil p/(d - s)\rceil \leq p/ (d - s)+1$, we immediately get $2^{-1}\sqrt{d - s}\cdot 2^{-p/(d - s)}\leq \mathrm{diam}(B)\leq 2\sqrt{d - s} \cdot 2^{-p / (d - s)}$. 
		Together, we have $\mathrm{diam}(A\times B) \leq \mathrm{diam}(A) +  \mathrm{diam}(B) \leq \sqrt{s} t^{-1} + 2\sqrt{d - s} \cdot 2^{-p / (d - s)}$ 
		and $\mathrm{diam}(A\times B) \geq ( \mathrm{diam}(A) +  \mathrm{diam}(B)) /2 \geq (\sqrt{d - s}\cdot 2^{-p/(d - s)} + \sqrt{s} t^{-1}) / 4$.
	\end{proof}

	\begin{lemma}\label{lem:zerocountsofV}
		For $\pi = \pi^{\text{priv}} \otimes \pi^{\text{pub}} $, let $\pi^{\text{priv}}$ be a histogram partition with $t$ bins and $\pi^{\text{pub}}$ be generated by Algorithm \ref{alg:partition} with depth $p$.  
		Then for each $i$, there are at most $ t^{ s} \cdot 2^{p - \lfloor p / (d - s) \rfloor \cdot \left( d - s - \sum_{\ell = s+1}^d W_i^{\ell}\right)}$ elements of $V_i^1, \cdots, V_i^{t^s \cdot 2^p}$ are non-zero. 
	\end{lemma}
	
	\begin{proof}[\textbf{Proof of Lemma \ref{lem:zerocountsofV}}]
		The number of non-zero elements is the cardinality of potential grids $|\mathcal{V}_i|$. 
		For $A\times B$, since the first $s$ dimensions are all private, all $t^s$ possibility of $A$ potentially appears in $\mathcal{V}_i$. 
		For $B$, consider the process of Algorithm \ref{alg:partition} that gradually grows $2^p$ grids. 
		At any step $k = 1,\cdots, p$, the algorithm split along the $\ell_k$-th dimension. 
		If $W_{i}^{\ell_k}$ is $1$, i.e. the feature is private, the number of potential grids will double. 
		Otherwise, it remains the same. 
		Each feature is split at least $\lfloor p / (d - s) \rfloor$ times, i.e. there are at most $p - \lfloor p / (d - s) \rfloor \cdot \left(d - s - \sum_{\ell = s + 1}^d W_i^{\ell}\right) $ splits which causes the number of potential grids to double. 
		This quantity is upper bonded by $ \sum_{\ell = s + 1}^d p\cdot W_i^{\ell}/ (d - s) + d - s$. 
		Multiplying potential possibilities of $A$ and $B$ leads to $ t^{ s} \cdot 2^{\sum_{\ell = s + 1}^d p\cdot W_i^{\ell}/ (d - s) + d - s}$. 
	\end{proof}

	\begin{proof}[\textbf{Proof of Lemma \ref{lem:boundingofprivatisederror}}]
		We intend to bound
		\begin{align}\nonumber
			\mathcal{R}_{L,\mathrm{P}}(\tilde{f}) - \mathcal{R}_{L,\mathrm{P}}(f) 
			= &  \int_{\mathcal{X}}\left|\tilde{f}(x) - f(x)\right|^2 d\mathrm{P}_X(x)\\
			= & \sum_{j}  \left|\frac{ \sum_{i}\tilde{Y}_i\cdot \tilde{V}_i^j}{ \sum_{i} \tilde{V}_i^j} - \frac{ \sum_{i}{Y}_i\cdot {V}_i^j}{ \sum_{i} {V}_i^j}\right|^2 \int_{A\times B} d\mathrm{P}_X(x)\nonumber
			\\
			= &  \sum_{j} \frac{1}{t^{s} 2^p}  \left|\frac{ \sum_{i}\tilde{Y}_i\cdot \tilde{V}_i^j}{ \sum_{i} \tilde{V}_i^j} - \frac{ \sum_{i}{Y}_i\cdot {V}_i^j}{ \sum_{i} {V}_i^j}\right|^2\label{equ:decompositionoferrorintocells}
		\end{align}
		where $j$ is the index of $A\times B$. 
		Here, $V_i^j$ is defined as $\eins_{A\times B}(X_i)$. 
		For each term, we have decomposition
		\begin{align*}
			& \left|\frac{ \sum_{i}\tilde{Y}_i\cdot \tilde{V}_i^j}{ \sum_{i} \tilde{V}_i^j} - \frac{ \sum_{i}{Y}_i\cdot {V}_i^j}{ \sum_{i} {V}_i^j}\right|^2  
			\leq   3 \cdot \Bigg(  \underbrace{\left|\frac{ \sum_{i=1}^n (\tilde{Y}_i - Y_i) \tilde{V}_i^j}{ \sum_{i=1}^n\tilde{V}^j_i}\right|^2}_{(I)} 
			\\
			& +  
			\underbrace{\left|\frac{ \sum_{i=1}^nY_i\tilde{V}^j_i - \sum_{i=1}^nY_i{V}^j_i }{\sum_{i=1}^n \tilde{V}^j_i }\right|^2}_{(II)}
			+ 
			\underbrace{\left|\frac{ \sum_{i=1}^nY_i{V}^j_i \sum_{i=1}^n {V}^j_i - \sum_{i=1}^nY_i{V}^j_i \sum_{i=1}^n \tilde{V}^j_i}{ \sum_{i=1}^n \tilde{V}^j_i \sum_{i=1}^n {V}^j_i}\right|^2}_{(III)}\Bigg)
		\end{align*}
		using triangular inequality.
		We bound the three parts separately. 
		
		(\textit{i}) For the numerator of $(I)$, let $v_j$ be the ratio of non-zero elements in $\tilde{V}_1^j,\cdots, \tilde{V}_n^j$, i.e. $v_j = \sum_{i = 1}^n \eins \{\tilde{V}_1^j \neq 0 \} / n$. 
		Note that $\tilde{Y}_i - Y_i = {4M}\xi_i / \varepsilon$ and $\tilde{V}_i^j\in [-1, 1]$.
		Thus, $(\tilde{Y}_i - Y_i)\tilde{V}_i^j$ are either zero or sub-exponential random variables with parameter $(\frac{2M }{\sqrt{2}\varepsilon},1)$.
		Consequently, Lemma \ref{lem:laplacebound} yields that
		\begin{align}\label{equ:privatized1}
			\left|\frac{1}{n}\sum_{i=1}^n(\tilde{Y}_i - Y_i)\tilde{V}^j_i \right|\leq \sqrt{\frac{128 M^2\cdot v_j \cdot \log n}{n\varepsilon^2}} 
		\end{align}
		with probability at least $ 1- 1/n^2$. 
		For the denominator, applying Lemma \ref{lem:hoeffding} yields 
		\begin{align}\label{equ:privatized3}
			\left|\frac{1}{n}\sum_{i=1}^n \tilde{V}_i^j - \frac{1}{n}\sum_{i=1}^n {V}_i^j \right| \leq \sqrt{\frac{v_j \log n}{n}}
		\end{align}
		with probability at least $1-2/n^2$. 
		Moreover, by Lemma \ref{lem:cellbound}, we have
		\begin{align*}
			\frac{1}{n}\sum_{i=1}^n V_i^j \geq    \int_{A\times B} d\mathrm{P}_X(x) - \left|  \int_{A\times B} d\mathrm{P}_X(x) -  \frac{1}{n}\sum_{i=1}^n V_i^j \right| \gtrsim \frac{1}{t^s \cdot 2^p} - \sqrt{\frac{\log n}{t^{s} \cdot 2^p \cdot n}} \gtrsim \frac{1}{t^s \cdot 2^p}
		\end{align*}
		holds with probability at least $1-1/n^2$. 
		Then, we can guarantee that 
		\begin{align}\label{equ:privatize2}
			\left|\frac{1}{n}\sum_{i=1}^n \tilde{V}_i^j \right| 
			\gtrsim   \frac{1}{t^s \cdot 2^p} - \sqrt{\frac{v_j\log n }{n }} \gtrsim  \frac{1}{t^s \cdot 2^p}
		\end{align}
		with probability $1-3/n^2$. 
		This together with \eqref{equ:privatized1} yields 
		\begin{align}\label{equ:privatizedconclusion1}
			(I)\lesssim \frac{t^{2s} \cdot 2^{2p}\cdot v_j   \cdot \log n }{n\varepsilon^2}.
		\end{align}
		
		(\textit{ii}) For $(II)$, we apply Lemma \ref{lem:hoeffding} and get 
		\begin{align*}
			\left|\frac{1}{n} \sum_{i=1}^n Y_i \tilde{V}_i^j - \frac{1}{n}\sum_{i=1}^n Y_i {V}_i^j \right| \leq \sqrt{\frac{ v_j \cdot M  \log n}{ n }}
		\end{align*}
		with probability $1-2/n^2$ since $|Y_i|\leq M$. 
		Combining this with \eqref{equ:privatize2}, we get 
		\begin{align}\label{equ:privatizedconclusion2}
			(II) \lesssim\frac{t^{2s} \cdot 2^{2p}\cdot v_j   \cdot \log n }{n\varepsilon^2}.
		\end{align}
		
		(\textit{iii}) For $(III)$, note that $V_i^j\in \{0,1\}$, we have $\sum_{i=1}^n Y_i V_i^j \leq M \sum_{i=1}^n V_i^j$. 
		Thus, 
		\begin{align*}
			\left|\frac{ \sum_{i=1}^nY_i{V}^j_i \sum_{i=1}^n {V}^j_i - \sum_{i=1}^nY_i{V}^j_i \sum_{i=1}^n \tilde{V}^j_i}{  \sum_{i=1}^n \tilde{V}^j_i \sum_{i=1}^n {V}^j_i}\right|^2 \leq & M^2 \left|\frac{ \sum_{i=1}^n {V}^j_i - \sum_{i=1}^n \tilde{V}^j_i}{ \sum_{i=1}^n \tilde{V}^j_i }\right|^2
			\\
			\lesssim & \frac{t^{2s} \cdot 2^{2p}\cdot v_j   \cdot \log n }{n\varepsilon^2}
		\end{align*}
		where the last inequality follows from \eqref{equ:privatized3} and \eqref{equ:privatize2}.
		Combining this with \eqref{equ:privatizedconclusion1} and \eqref{equ:privatizedconclusion2}, we have 
		\begin{align*}
			\sum_{j} \frac{1}{t^{s} 2^p}  \left|\frac{ \sum_{i}\tilde{Y}_i\cdot \tilde{V}_i^j}{ \sum_{i} \tilde{V}_i^j} - \frac{ \sum_{i}{Y}_i\cdot {V}_i^j}{ \sum_{i} {V}_i^j}\right|^2  \lesssim   \frac{t^{s} \cdot 2^{p}\cdot \sum_{j} v_j   \cdot \log n }{n\varepsilon^2}
		\end{align*}
		with probability $\mathrm{P}^n \times \mathrm{R}^n$ at least $1- 4/n^2$. 
		By Lemma \ref{lem:zerocountsofV}, 
		\begin{align*}
			\sum_{j} v_j  \leq \frac{1}{n} \sum_{i = 1}^n \sum_j \eins\{V_i^j\neq 0\} \lesssim   &   \frac{t^{ s}\log n }{n\varepsilon^2}\sum_{i=1}^n 2^{ \sum_{\ell = s + 1}^d  W_i^{\ell} \cdot p / (d - s) }.
		\end{align*}
		This directly implies the desired bound of $\mathcal{R}_{L,\mathrm{P}}(f_{\pi}^{\mathrm{DP}}) - \mathcal{R}_{L,\mathrm{P}}(f_{\pi})$. 
	\end{proof}
	
	\begin{proof}[\textbf{Proof of Lemma \ref{lem:boundingofsampleerror}}]
		We intend to bound 
		\begin{align*}
			\mathcal{R}_{L,\mathrm{P}}(f) - \mathcal{R}_{L,\mathrm{P}}(\overline{f}) 
			\leq  & \max_{j}  \left|  \frac{ \frac{1}{n}\sum_{i=1}^nY_i{V}_i^j}{ \frac{1}{n}\sum_{i=1}^n {V}_i^j} - \frac{\int_{A\times B}f^*(x')d\mathrm{P}_X(x')}{\int_{A\times B }d\mathrm{P}_X(x)}\right|^2 \\
			\leq & \underbrace{\left|  \frac{ \frac{1}{n}\sum_{i=1}^nY_i{V}_i^j\int_{A\times B}d\mathrm{P}_X(x') - \int_{A\times B}f^*(x')d\mathrm{P}(x') \int_{A\times B }d\mathrm{P}_X(x')}{ \frac{1}{n}\sum_{i=1}^n {V}_i^j \int_{A\times B}d\mathrm{P}_X(x)} \right|^2}_{(I)} \\
			+ &  \underbrace{\left|  \frac{  \int_{A\times B }f^*(x')d\mathrm{P}(x') \int_{A\times B}d\mathrm{P}_X(x') -  \frac{1}{n}\sum_{i=1}^n{V}_i^j\int_{A\times B }f^*(x')d\mathrm{P}_X(x')}{ \frac{1}{n}\sum_{i=1}^n {V}_i^j \int_{A\times B}d\mathrm{P}_X(x)} \right|^2}_{(II)}
		\end{align*}
		We bound two terms separately. For $(I)$, Lemma \ref{lem:cellbound} yields
		\begin{align}\label{equ:sampleboundofdenominator}
			\frac{1}{n}\sum_{i=1}^n V_i^j \int_{A_j}d\mathrm{P}_X(x) \gtrsim t^{-2s} \cdot 2^{-2 p} 
		\end{align}
		with probability $1-1/n^2 $. 
		For the numerator, Lemma \ref{lem:cellboundY} yields
		\begin{align*}
			\left| \frac{1}{n} \sum_{i=1}^n Y_i V_i^j - \int_{A\times B }f^*(x')d\mathrm{P}_X(x')\right|^2 \left| \int_{A\times B}d\mathrm{P}_X(x')\right|^2  \lesssim \frac{\log n }{t^s 2^p\cdot n} \cdot t^{-2s} \cdot 2^{-2p}. 
		\end{align*}
		Together, we get $(I)\lesssim {t^s 2^p}/{n}$. 
		Analogously, by Lemma \ref{lem:cellbound}, we have
		\begin{align*}
			\left|  \int_{A\times B }d\mathrm{P}_X(x') -  \sum_{i=1}^n{V}_i^j\right|^2\left|\int_{A\times B}f^*(x')d\mathrm{P}_X(x')\right|^2\leq \frac{1}{t^s 2^{p}\cdot n } \cdot M^2\cdot t^{-2s } \cdot 2^{-2p}.
		\end{align*}
		This together with \eqref{equ:sampleboundofdenominator} yields 
		$(II)\lesssim t^s 2^p/n$. The bound of $(I)$ and $(II)$ together yields the desired conclusion. 
	\end{proof}

	\begin{proof}[\textbf{Proof of Lemma \ref{lem:boundingofapproximationerror}}]
		Let $A\times B (x)$ be the grid of which $x$ belongs to.  
		We intend to bound 
		\begin{align*}
			\mathcal{R}_{L,\mathrm{P}}(\overline{f})  - \mathcal{R}_{L,\mathrm{P}}(f^*) = 
			\int_{x\in\mathcal{X}}\left( \frac{\int_{A\times B (x)}f(x')d\mathrm{P}_X(x')}{\int_{A\times B(x)}d\mathrm{P}_X(x')} - f(x)\right)^2 d\mathrm{P}_X(x).  
		\end{align*}
		For each $x$, if Assumption \ref{asp:alphaholder} holds, the point-wise error can be bounded by 
		\begin{align*}
			\frac{\int_{A\times B(x)}f(x')d\mathrm{P}_X(x')}{\int_{A\times B(x)}d\mathrm{P}_X(x')} - f(x) \leq & \frac{\int_{A\times B(x)}|f(x') - f(x)|d\mathrm{P}_X(x')}{\int_{A \times B (x)}d\mathrm{P}_X(x')} \\ 
			\leq &  \frac{\int_{A \times B (x)}c_L\|x' - x\|^{\alpha}d\mathrm{P}_X(x')}{\int_{A \times B (x)}d\mathrm{P}_X(x')}  \leq c_L \mathrm{diam} (A \times B (x))^{\alpha}
		\end{align*}
		Then we can bound the excess risk using Lemma \ref{lem::treeproperty}, 
		\begin{align*}
			\mathcal{R}_{L,\mathrm{P}}(\overline{f})  - \mathcal{R}_{L,\mathrm{P}}(f^*) \lesssim  \mathrm{diam} (A \times B (x))^{\alpha} \lesssim 2^{- 2 p\alpha / (d - s)} + t^{- 2 \alpha }. 
		\end{align*}
		
	\end{proof}

	\section{Experiment Details}\label{app:experiments}

	\subsection{Implementation Details}
	\label{sec:implementationdetail}
	
	For \texttt{HistOfTree}, we adopt the generalized randomized response mechanism introduced in Section \ref{app:generelizedrandomresponse}. 
	We add one more parameter adjusting the allocation of the privacy budget on the numerator and denominator. 
	Specifically, let the mechanism for $\tilde{Y}_i$ and $\tilde{U}_i$ be respectively $\rho\varepsilon$-LDP and $(1-\rho)\varepsilon$-semi-feature LDP for $\rho \in [0,1] $. 
	In this case, the mechanisms are which means the hybrid mechanism is still $\varepsilon$-LDP. 
	We select $\rho \in \{0.5,0.7,0.9\}$. 
	For other parameters, we select $p\in \{1, 2, 4, 6\}$ and $t\in \{1,2,3\}$.

	For \texttt{AdHistOfTree}, we also adopt the generalized randomized response mechanism.
	For better empirical performance, the second term of the objective function is multiplied with a constant selected in $\{0.01,0.1,1\}$. 
	Also, we allow the selection $t\in\{t^*-1, t^*, t^* + 1\}$, where $t^*$ is the default value derived in Section \ref{sec:selectionofs}.

	For \texttt{Hist}, we implement the private histogram proposed by \citep{berrett2021strongly, gyorfi2022rate} in Python. 
	\texttt{Hist} applies the Laplacian mechanism to privatize the estimation of marginal and joint probabilities for a cubic histogram partition with the number of bins $t$. 
	In cells with a private marginal probability estimation less than $\zeta$, estimation is truncated to 0.
	We let $t\in\{1,2,3,4\}$. 
	We set the truncation parameter $\zeta \in \{0.01, 0.05\}$.
	
	For \texttt{ParDT} and \texttt{LabelDT}, we add Laplace noise with scale $\xi /\varepsilon$ where $\xi$ is the range of the label. 
	For \texttt{ParDT},  \texttt{LabelDT}, \texttt{KRR}  and \texttt{DT}, we use the implementation by Scikit-Learn \citep{scikit-learn}. We select \textit{$max\_depth$} in $\{1,2,4, 6,8\}$ and \textit{$min\_sample\_leaf$} in $\{1, 10, 100\}$. 
	For  \texttt{KRR}, we select $k$ in $\{2, 3, 4, 5\}$.

	\subsection{Dataset Details}\label{app:datasets}
	
	For all datasets, each feature is min-max scaled to the range $[0,1]$ individually. 
	
	ABA: The \textit{Abalone} dataset originally comes from biological research \citep{nash1994population} and now it is accessible on UCI Machine Learning Repository \citep{Dua:2019}. ABA contains $4177$ observations of one target variable and $8$ attributes related to the physical measurements of abalone. 
	
	AQU: The \textit{
		QSAR aquatic toxicity} dataset was used to develop quantitative regression QSAR models to predict acute aquatic toxicity towards the fish Pimephales promelas (fathead minnow) on a set of 908 chemicals. It contains $546$ instances of $8$ input attributes and $1$ output attribute.

	BUI: The \textit{Residential Building Data Set Data Set} dataset on UCI Machine Learning Repository includes construction cost, sale prices, project variables, and economic variables corresponding to real estate single-family residential apartments in Tehran, Iran. It contains $372$ instances of $103$ input attributes and $2$ output attributes.

	DAK: The \textit{Istanbul Stock Exchange} dataset includes returns of the Istanbul stock exchange with seven other international indexes. It has 536 instances with 10 features.

	FIS: The \textit{QSAR fish toxicity} dataset on UCI Machine Learning Repository was used to develop quantitative regression QSAR models to predict acute aquatic toxicity towards the fish Pimephales promelas (fathead minnow) on a set of 908 chemicals. It contains $908$ instances of $7$ features.

	MUS: The \textit{
		The Geographical Original of Music} dataset was built from a personal collection of 1059 tracks covering 33 countries/areas. The program MARSYAS \citep{zhou2014predicting} was used to extract audio features from the wave files. We used the default MARSYAS settings in single vector format (68 features) to estimate the performance with basic timbal information covering the entire length of each track.

	POR: The \textit{Stock Portfolio Performance} data set of performances of weighted scoring stock portfolios are obtained with mixture design from the US stock market historical database \citep{liu2017using}.
	It has 315 samples with 12 features. 
	
	PYR: The \textit{Pyrimidines} dataset is a subset of the Qualitative Structure Activity Relationships dataset on UCI datasets.
	This sub-dataset has 74 instances of dimension 27
	
	RED: This dataset contains the information on red wine of the \textit{
		Wine Quality} dataset \citep{cortez2009modeling} on UCI Machine Learning Repository. There are $11$ input variables to predict the output variable wine quality. $4898$ instances are collected in the dataset.

	WHI: This dataset also originates from the \textit{Wine Quality} dataset \citep{cortez2009modeling} on the UCI Machine Learning Repository. There are $11$ features related to white wine to predict the corresponding wine quality.

	\subsection{Additional Experiment Results}\label{app:additionalresults}
	
	\begin{table}[htbp]
		\centering
		\caption{Real data performance for $\varepsilon = 1$. 
		}
		\label{tab:realdataepsilon1}
		\resizebox{0.95\linewidth}{!}{
			\renewcommand{\arraystretch}{1}
			\setlength{\tabcolsep}{7pt}
			\begin{tabular}{l|lllll|lllllll}
				\toprule
				\multicolumn{1}{c|}{\multirow{3}{*}{}}                                  & \multicolumn{5}{c|}{Aligned}                                                                                                                                  & \multicolumn{7}{c}{Personalized}                                                                                                                                                                                 \\ 
				\cmidrule(r){2-6} \cmidrule(l){7-13}
				\multicolumn{1}{c|}{}                                                   & \multicolumn{2}{c}{\texttt{HistOfTree}}           & \multicolumn{1}{c}{\multirow{2}{*}{\texttt{ParDT}}} & \multicolumn{1}{c}{\multirow{2}{*}{\texttt{Hist}}} &  \multicolumn{1}{c|}{\multirow{2}{*}{\texttt{KRR}}} & \multicolumn{2}{c}{\texttt{HistOfTree}}           & \multicolumn{2}{c}{\texttt{AdHistOfTree}}         & \multicolumn{1}{c}{\multirow{2}{*}{\texttt{ParDT}}} & \multicolumn{1}{c}{\multirow{2}{*}{\texttt{Hist}}}  & \multicolumn{1}{c}{\multirow{2}{*}{\texttt{KRR}}} \\
				\multicolumn{1}{c|}{}                                                   & \multicolumn{1}{c}{ME} & \multicolumn{1}{c}{CART} & \multicolumn{1}{c}{}                                & \multicolumn{1}{c}{}            &                   & \multicolumn{1}{c}{ME} & \multicolumn{1}{c}{CART} & \multicolumn{1}{c}{ME} & \multicolumn{1}{c}{CART} & \multicolumn{1}{c}{}                                & \multicolumn{1}{c}{}                               \\ \midrule
				ABA                                                                     & 2.04                   & \textbf{1.80*}                    & 2.04                                                & 2.02               &       2.74                           & \textbf{1.85*}                   & 2.04                     & 1.86                   & 2.00                     & 2.26                                                & 2.02                      &  2.71                         \\
				AQU                                                                     & \textbf{1.71*}                   & \textbf{1.71*}                     & 2.86                                                & 1.80                          & 3.35                      & 1.71                   & 1.73                     &\textbf{1.69*}                   & \textbf{1.69*}                     & 3.03                                                & 1.80                     & 3.35                          \\
				BUI                                                                     & \textbf{1.59*}                   & \textbf{1.59*}                     & 3.18                                                & 5e5                     & 6.70                            & 1.64                   & 1.64                     & \textbf{1.61*}                   & \textbf{1.61*}                     & 3.28                                                & 5e5                        & 6.57                       \\
				DAK                                                                     & 6.69                   &\textbf{6.65*}                    & 8.03                                                & 8.43                     & 22.6                           & 7.95                   & 7.95                     & \textbf{7.47*}                   & \textbf{7.47*}                     & 19.0                                                & 8.43                          & 22.6                     \\
				FIS                                                                     & \textbf{2.22*}                   & 2.32                     & 3.22                                                & 2.45                       & 4.60                         & 2.32                   & 2.32                     & \textbf{2.26*}                   & \textbf{2.26*}                     & 3.49                                                & 2.45                               & 4.59                \\
				MUS                                                                     & {1.17}                   & \textbf{1.17}                     & 1.89                                                & 9e5                     & 2.72                            & 1.17                   & 1.17                     & \textbf{1.16}                   & \textbf{1.16}                     & 1.99                                                & 9e5                                   & 2.74             \\
				POR                                                                     &\textbf{4.15*}                   &\textbf{4.15*}                    & 4.67                                                & 7.22                      & 7.89                          & 4.04                  & 4.04                     & \textbf{3.68*}                   & \textbf{3.68*}                     & 4.72                                                & 7.22                               & 7.70              \\
				PYR                                                                     & \textbf{2.31*}                   & \textbf{2.31*}                     & 2.98                                                & 1e5                          & 5.61                       & 2.32                   & 2.32                     & \textbf{1.75*}                   & \textbf{1.75*}                     & 3.18                                                & 1e5                                & 5.64                \\
				RED                                                                     & \textbf{1.48*}                  &  \textbf{1.48*}                  & 2.02                                                & 1.81                            & 2.53                    & \textbf{1.48*}                   & 1.52                     & \textbf{1.48*}                   & \textbf{1.48*}                     & 2.08                                                & 1.81                      & 2.53                         \\
				WHI                                                                     & \textbf{1.46*}                   & \textbf{1.46*}                     & 1.61                                                & 1.61                          & 1.85                      & 1.47                   & 1.47                     & \textbf{1.46}                   & \textbf{1.46}                     & 1.70                                                & 1.61                               &         1.84        \\ \midrule
				\multicolumn{1}{c|}{\begin{tabular}[c]{@{}c@{}}Rank\\ Sum\end{tabular}} &        21                &      \textbf{18}                    &                  35                                   &                          38              &47             &          26              &           34               &          \textbf{11}              &           12               &                         56                            &                56 & 67                                    \\ \bottomrule
			\end{tabular}
		}
	\end{table}
	
	\begin{table}[htbp]
		\centering
		\caption{Real data performance for $\varepsilon = 4$. 
		}
		\label{tab:realdataepsilon4}
		\resizebox{0.95\linewidth}{!}{
			\renewcommand{\arraystretch}{1}
			\setlength{\tabcolsep}{7pt}
			\begin{tabular}{l|lllll|lllllll}
				\toprule
				\multicolumn{1}{c|}{\multirow{3}{*}{}}                                  & \multicolumn{5}{c|}{Aligned}                                                                                                                                  & \multicolumn{7}{c}{Personalized}                                                                                                                                                                                 \\ 
				\cmidrule(r){2-6} \cmidrule(l){7-13}
				\multicolumn{1}{c|}{}                                                   & \multicolumn{2}{c}{\texttt{HistOfTree}}           & \multicolumn{1}{c}{\multirow{2}{*}{\texttt{ParDT}}} & \multicolumn{1}{c}{\multirow{2}{*}{\texttt{Hist}}} &  \multicolumn{1}{c|}{\multirow{2}{*}{\texttt{KRR}}} & \multicolumn{2}{c}{\texttt{HistOfTree}}           & \multicolumn{2}{c}{\texttt{AdHistOfTree}}         & \multicolumn{1}{c}{\multirow{2}{*}{\texttt{ParDT}}} & \multicolumn{1}{c}{\multirow{2}{*}{\texttt{Hist}}}  & \multicolumn{1}{c}{\multirow{2}{*}{\texttt{KRR}}} \\
				\multicolumn{1}{c|}{}                                                   & \multicolumn{1}{c}{ME} & \multicolumn{1}{c}{CART} & \multicolumn{1}{c}{}                                & \multicolumn{1}{c}{}            &                   & \multicolumn{1}{c}{ME} & \multicolumn{1}{c}{CART} & \multicolumn{1}{c}{ME} & \multicolumn{1}{c}{CART} & \multicolumn{1}{c}{}                                & \multicolumn{1}{c}{}                               \\ \midrule
				ABA                                                                     & 1.72                   & 1.59                     & \textbf{1.58}                                                & 1.97         & 1.64                                       & \textbf{1.54}                   & 1.63                     & 1.55                   & 1.62                     & 1.61                                                & 1.97                         &1.97                      \\
				AQU                                                                     & 1.45                   & \textbf{1.35*}                     & 1.56                                                & 1.58             & 1.64                                   & 1.50                   & 1.47                     & 1.54                   & \textbf{1.45*}                     & 1.68                                                & 1.58                   & 1.65                            \\
				BUI                                                                     & 1.23                   & \textbf{1.16*}                     & 1.32                                                & 3e4             &     1.30                               & 1.26                   & \textbf{1.19*}                     & 1.28                   & 1.22                     & 1.36                                                & 3e4                    & 1.29                            \\
				DAK                                                                     & 4.67                   & \textbf{4.48*}                     & 8.03                                                & 6.32                & 6.57                                & \textbf{4.53}                   & 4.55                     & 5.95                   & 5.57                     & 6.01                                                & 6.32                      & 6.58                         \\
				FIS                                                                     & 1.66                   & \textbf{1.64}                     & 1.65                                                & 2.12                       & 1.82                         & \textbf{1.62*}                   & 1.66                     & 1.77                   & 1.68                     & 1.68                                                & 2.12                         & 1.77                      \\
				MUS                                                                     & \textbf{1.12}                  & \textbf{1.12}                     & 1.15                                                & 6e3                           & 1.36                      & \textbf{1.10}                   & 1.12                     & \textbf{1.10}                   & 1.12                     & 1.22                                                & 6e3                             & 1.14                   \\
				POR                                                                     & 2.89                   & \textbf{2.40*}                     & 2.98                                                & 3.21                          & 3.09                      & \textbf{2.66}                   & 2.75                     & 2.68                   & 2.72                     & 2.97                                                & 3.31                           & 3.02                    \\
				PYR                                                                     & 1.26                   & \textbf{1.16*}                     & 1.33                                                & 7e2                            & 1.49                     & 1.29                   & 1.29                     & \textbf{1.27}                   & \textbf{1.27}                  & 1.43                                                & 7e2                              & 1.50                  \\
				RED                                                                     & \textbf{1.27}                   & 1.47                     & 1.29                                                & 1.49                              & 1.44                  & 1.27                   & \textbf{1.26}                     & 1.30                   & 1.27                     & 1.28                                                & 1.49                               & 1.44                \\
				WHI                                                                     & \textbf{1.24}                   & \textbf{1.24}                    & 1.26                                                & 1.51                            & 1.42                    & 1.24                   & \textbf{1.23}                     & 1.43                   & 1.37                     & 1.30                                                & 1.51                           & 1.42                    \\ \midrule
				\multicolumn{1}{c|}{\begin{tabular}[c]{@{}c@{}}Rank\\ Sum\end{tabular}} &       22                 &         \textbf{16}                 &                      29                               &                            47 & 38                         &         \textbf{18}               &            24              &        34                &         26                 &                        47                             &              66 & 57                                      \\ \bottomrule
			\end{tabular}
		}
	\end{table}

\end{document}